\pdfoutput=1
\documentclass{article}

\PassOptionsToPackage{table}{xcolor}
\usepackage{times}
\let\htmlEmailSize\footnotesize
\makeatletter
\renewcommand{\normalsize}{\@setfontsize\normalsize{10pt}{11pt}}
\renewcommand{\small}{\@setfontsize\small{9pt}{10pt}}
\renewcommand{\footnotesize}{\@setfontsize\footnotesize{9pt}{10pt}}
\renewcommand{\scriptsize}{\@setfontsize\scriptsize{7pt}{8pt}}
\renewcommand{\large}{\@setfontsize\large{12pt}{14pt}}
\renewcommand{\LARGE}{\@setfontsize\LARGE{17pt}{20pt}}
\normalsize
\renewcommand{\section}{\@startsection{section}{1}{\z@}{-2ex plus -0.5ex minus -0.2ex}{1.5ex plus 0.3ex minus 0.2ex}{\large\scshape\raggedright}}
\renewcommand{\subsection}{\@startsection{subsection}{2}{\z@}{-1.8ex plus -0.5ex minus -0.2ex}{0.8ex plus 0.2ex}{\normalsize\scshape\raggedright}}
\renewcommand{\subsubsection}{\@startsection{subsubsection}{3}{\z@}{-1.5ex plus -0.5ex minus -0.2ex}{0.5ex plus 0.2ex}{\normalsize\scshape\raggedright}}
\renewcommand{\maketitle}{\par\begingroup\centering{\LARGE\scshape\@title\par}\ifx\@author\@empty\else\vspace{12pt}{\normalsize\pdfauthorblock\par}\fi\endgroup\vspace{12pt}\setcounter{footnote}{0}}
\renewenvironment{abstract}{\vspace{0.075in}\centerline{\large\scshape Abstract}\vspace{0.5ex}\begin{quote}}{\par\end{quote}\vspace{1ex}}
\makeatother
\usepackage[utf8]{inputenc}
\usepackage[T1]{fontenc}
\usepackage{hyperref,url,booktabs,array,multirow,adjustbox}

\usepackage{microtype,placeins,wrapfig}
\usepackage{mathtools}
\usepackage{xcolor}
\usepackage[authoryear,round]{natbib}
\usepackage{bookmark}
\hypersetup{hidelinks}
\usepackage{refcount}
\usepackage{graphicx}
\usepackage{booktabs}
\usepackage{caption}
\usepackage{subcaption}
\usepackage{amsmath}
\usepackage{amssymb}
\usepackage{amsthm}

\usepackage{float}
\usepackage{algorithm}
\IfFileExists{algpseudocode.sty}{
  \usepackage{algpseudocode}
}{
  \newcounter{algorithmicline}

}
\makeatletter
\@ifundefined{STATE}{}{}
\@ifundefined{FOR}{}{}
\@ifundefined{ENDFOR}{}{}
\@ifundefined{IF}{}{}
\@ifundefined{ENDIF}{}{}
\makeatother
\usepackage{xparse}
\usepackage{aliascnt}
\IfFileExists{cleveref.sty}{
  \usepackage[nameinlink,noabbrev]{cleveref}
}{
  \newcommand{\cref}[1]{\autoref{##1}}
  \let\Cref\cref
  \newcommand{\crefname}[3]{}
  \newcommand{\Crefname}[3]{}
}

\theoremstyle{plain}
\newtheorem{theorem}{Theorem}[section]

\newaliascnt{lemma}{theorem}
\newtheorem{lemma}[lemma]{Lemma}
\aliascntresetthe{lemma}

\newaliascnt{proposition}{theorem}
\newtheorem{proposition}[proposition]{Proposition}
\aliascntresetthe{proposition}

\newaliascnt{corollary}{theorem}
\newtheorem{corollary}[corollary]{Corollary}
\aliascntresetthe{corollary}

\theoremstyle{definition}
\newaliascnt{definition}{theorem}
\newtheorem{definition}[definition]{Definition}
\aliascntresetthe{definition}

\newaliascnt{example}{theorem}

\aliascntresetthe{example}

\newaliascnt{assumption}{theorem}
\newtheorem{assumption}[assumption]{Assumption}
\aliascntresetthe{assumption}

\newcounter{remark}[theorem]
\NewDocumentEnvironment{remark}{o}
{
  \refstepcounter{remark}
  \par\vspace*{-\parskip}\noindent\textbf{Remark \thetheorem.\arabic{remark}.}\IfValueT{#1}{\ (#1).}\ \normalfont\ignorespaces
}
{
  \par\ignorespacesafterend
}

\crefname{equation}{equation}{equations}
\Crefname{equation}{Equation}{Equations}
\crefname{theorem}{theorem}{theorems}
\Crefname{theorem}{Theorem}{Theorems}
\crefname{assumption}{assumption}{assumptions}
\Crefname{assumption}{Assumption}{Assumptions}
\crefname{definition}{definition}{definitions}
\Crefname{definition}{Definition}{Definitions}
\crefname{example}{example}{examples}
\Crefname{example}{Example}{Examples}
\crefname{lemma}{lemma}{lemmas}
\Crefname{lemma}{Lemma}{Lemmas}
\crefname{corollary}{corollary}{corollaries}
\Crefname{corollary}{Corollary}{Corollaries}
\crefname{proposition}{proposition}{propositions}
\Crefname{proposition}{Proposition}{Propositions}
\crefname{algorithm}{algorithm}{algorithms}
\Crefname{algorithm}{Algorithm}{Algorithms}
\crefname{remark}{remark}{remarks}
\Crefname{remark}{Remark}{Remarks}
\crefname{logicpart}{Part}{Parts}
\Crefname{logicpart}{Part}{Parts}
\crefname{logicsubpart}{Part}{Parts}
\Crefname{logicsubpart}{Part}{Parts}

\newaliascnt{traceproof}{theorem}
\aliascntresetthe{traceproof}
\crefname{traceproof}{proof}{proofs}
\Crefname{traceproof}{Proof}{Proofs}
\NewDocumentEnvironment{traceproof}{o m}
{
  \refstepcounter{traceproof}
  \IfNoValueF{#1}{\label{#1}}
  \begin{proof}[\textup{\textbf{Proof \thetraceproof}\ (#2)}]
}
{
  \end{proof}
}

\newcommand{\mypara}[1]{\paragraph{#1}}

\newcounter{logicpart}
\newcounter{logicsubpart}[logicpart]

\newcommand{\Acal}{\mathcal{A}}
\newcommand{\Bcal}{\mathcal{B}}

\newcommand{\Ecal}{\mathcal{E}}
\newcommand{\Fcal}{\mathcal{F}}

\newcommand{\Hcal}{\mathcal{H}}

\newcommand{\Lcal}{\mathcal{L}}

\newcommand{\Ncal}{\mathcal{N}}

\newcommand{\Rcal}{\mathcal{R}}

\newcommand{\Ucal}{\mathcal{U}}

\newcommand{\Zcal}{\mathcal{Z}}

\newcommand{\Rbb}{\mathbb{R}}

\newcommand{\argmin}{\operatorname*{arg\,min}}

\setcitestyle{authoryear,round,citesep={;},aysep={,},yysep={;}}
\hypersetup{colorlinks=true,allcolors=black,pdftitle={Ontic State Recovery through Prediction},pdfauthor={Ping Zhang, Quang-Huy Nguyen, Yingtie Lei, Wei-Lun Chao}}
\renewcommand\mypara[1]{\par\addvspace{0.6mm}\noindent\strut\textbf{#1}\enspace\ignorespaces}
\title{Ontic State Recovery through Prediction}
\author{
\begin{tabular}{ccc}
Ping~Zhang\textsuperscript{1} & \rule{90pt}{0pt} & Quang-Huy~Nguyen\textsuperscript{1}\\
\mbox{\htmlEmailSize\href{mailto:pingzh@bu.edu}{pingzh@bu.edu}} & & \mbox{\htmlEmailSize\href{mailto:quanghuy@bu.edu}{quanghuy@bu.edu}}\\
\rule{0pt}{18pt} & & \\
Yingtie~Lei\textsuperscript{2} & & Wei-Lun~Chao\textsuperscript{1}\\
\mbox{\htmlEmailSize\href{mailto:lei.441@osu.edu}{lei.441@osu.edu}} & & \mbox{\htmlEmailSize\href{mailto:chao209@bu.edu}{chao209@bu.edu}}\\
\end{tabular}
\\[11pt]
\textsuperscript{1}Boston University\quad\textsuperscript{2}The Ohio State University
}
\newcommand{\pdfauthorblock}{
  \large
  \makebox[.48\textwidth]{Ping~Zhang\textsuperscript{1}}\makebox[.48\textwidth]{Quang-Huy~Nguyen\textsuperscript{1}}\\[0pt]
  {\small\rmfamily\color{black!72}\makebox[.48\textwidth]{pingzh@bu.edu}\makebox[.48\textwidth]{quanghuy@bu.edu}}\\[8pt]
  \makebox[.48\textwidth]{Yingtie~Lei\textsuperscript{2}}\makebox[.48\textwidth]{Wei-Lun~Chao\textsuperscript{1}}\\[0pt]
  {\small\rmfamily\color{black!72}\makebox[.48\textwidth]{lei.441@osu.edu}\makebox[.48\textwidth]{chao209@bu.edu}}\\[11pt]
  {\normalsize\color{black!85}\mbox{\textsuperscript{1}Boston University\enspace\textsuperscript{2}The Ohio State University}}
}

\date{}

\begin{document}
\maketitle
\addtocontents{toc}{\protect\setcounter{tocdepth}{-1}}

\begin{abstract}
  We encounter a system through its responses, yet its state concerns how the system is. Predicting further responses promises knowledge beyond what is immediately observed. What that knowledge tells us about the system's state depends on what we mean by a state in the first place. We give this notion a model-relative meaning: a state is the complete response profile over a specified query universe, and configurations agreeing on every query are regarded as the same state. From this perspective, a learned state is understood through the distinctions between these profiles that it retains. Prediction concerns selected observations; recovery concerns the profile as a whole. We examine their relationship through query ambiguity, state-space geometry, and the choice of prediction targets. Synthetic systems with known response laws and neural recordings provide complementary settings for this investigation. The aim is to clarify how prediction can become a way of knowing a system, and what it would mean for a learned state to preserve the state being studied.
\end{abstract}

\section{Introduction}
\label{sec:paper-introduction}

\begin{wrapfigure}{r}{.55\textwidth}
  \centering
  \noindent\makebox[\linewidth][c]{\hspace*{2bp}\makebox[204.78696pt][c]{\raisebox{-20.00000pt}[30.97462pt][0.00000pt]{\includegraphics[width=244.78696pt,height=70.97462pt]{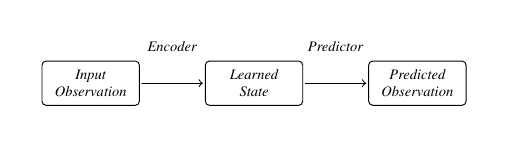}}}
\hspace*{-2bp}}
  \caption{From \textit{observation} to prediction. The \textit{encoder} maps \textit{input observations} to a \textit{learned state}; the \textit{predictor} estimates \textit{observations} at the \textit{target queries}.}
  \label{fig:paper-prediction-pipeline}
\end{wrapfigure}

Learning a system's state requires deciding which differences between its possible conditions matter. Knowing that state would let us reason beyond the aspects we happen to observe. For an animal, a neural recording offers one such view: a pattern of activity from which we seek to understand the condition that gives rise to it~\citep{Cunningham2014,Churchland2012}. Giving this underlying condition a precise meaning is part of the problem of knowing it.

Predicting further observations offers a way to study this problem. Neural population models learn to predict held-out activity~\citep{Pei2021NLB}, and models trained across recordings use related objectives to support downstream decoding~\citep{Ye2023NDT2}. A learned state is useful here because it carries information about observations beyond its input.

\begin{center}
  \parbox{\linewidth}{\centering\emph{``When does learning to predict a system amount to learning its state?''}}
\end{center}

We make this question precise relative to a query universe. A \emph{query} specifies what is measured, and an \emph{observation} is the resulting value. Within that universe, the \emph{ontic state} is the complete response profile: configurations agreeing on every query count as the same state. In a neural recording, a query specifies a unit and a trailing-window duration. An \textit{encoder} maps input observations to a learned state, from which a \textit{predictor} estimates observations at target queries (\cref{fig:paper-prediction-pipeline}). Theoretical work relates auxiliary prediction and context learning to downstream performance~\citep{Lee2021ProvableSSL,Zhai2025Contextures}. We ask how much of the ontic state can be recovered from the learned state that supports these predictions.

Consider targets that all depend on one coordinate of a state. Exact prediction constrains that coordinate while leaving the others free. \emph{Query ambiguity} measures the unresolved state differences. It connects the distinctions tested by prediction to the profile we seek to recover.

This connection motivates three contributions:
\begin{enumerate}
  \item \textit{Ontic-state formulation.} We formulate ontic states as complete response profiles over a query universe.
  \item \textit{Recovery guarantees.} We establish prediction--recovery guarantees in both general settings and the low-dimensional settings commonly studied in neuroscience.
  \item \textit{Fixed-family sharpness.} We prove matching budget bounds on a fixed state family, distinguishing targets shared across \textit{encoders} from targets tailored to each \textit{encoder}.
\end{enumerate}

\section{Ontic State}\label{sec:paper-setup}

\subsection{State, Query, and Observation}

\begin{wrapfigure}{l}{.55\textwidth}
  \centering
  \noindent\makebox[\linewidth][c]{
    \hspace*{1.75bp}\makebox[219.23898pt][c]{\raisebox{-20.00000pt}[62.03368pt][0.00000pt]{\includegraphics[width=259.23898pt,height=102.03368pt]{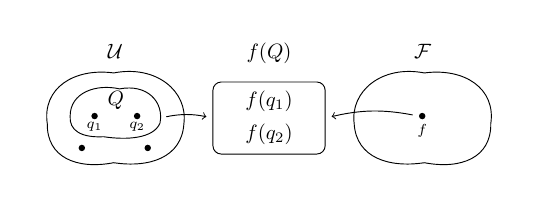}}}
\hspace*{-1.75bp}
  }
  \caption{\textit{State}, \textit{query}, and \textit{observation}. The \textit{state} $f\in\Fcal$ and \textit{query set} $Q=\{q_1,q_2\}\subseteq\Ucal$ determine the \textit{observation vector} $f(Q)$. Its entries are $f(q_1)$ and $f(q_2)$; unlabeled points denote queries outside $Q$.}
  \label{fig:ontic-state-observation}
\end{wrapfigure}

What does it mean for a system to be in a particular state, independently of how we describe it? We need a criterion of identity. Leibniz's account of individual substance locates identity in a complete concept determining the predicates of its subject~\citep{leibniz1686discourse}. For a scientific account of state, which properties should determine this identity?

We give these properties empirical meaning through the observations the system could produce~\citep{peirce1878ideas,james1907pragmatism}. Each query specifies one such observation. Two configurations may agree in position and differ in velocity: whether we distinguish them depends on the queries admitted. We therefore fix a query universe before defining its states.

If no question in this universe can distinguish two putative states, we regard them as one ontic state: a difference in description alone does not constitute a difference in state. Their common response profile fixes this identity independently of the coordinates used to represent it~\citep{poincare1905hypothesis}. An experiment reveals part of the profile, while the state specifies responses throughout the universe. We thus model a state as a function from queries to observations.

\begin{definition}[Ontic State]\label{def:ontic-state-space}
  An \emph{ontic state} is an admissible bounded function
  \begin{equation}
    f:\Ucal\to\Rbb,
  \end{equation}
  where $\Ucal$ is a fixed nonempty set. The collection $\Fcal$ of all admissible ontic states is called the \emph{ontic state space}.
\end{definition}

\begin{definition}[Query Universe]\label{def:ontic-query-space}
  The common domain $\Ucal$ of the ontic states is called the \emph{query universe}, and its elements are called \emph{queries}. For $f\in\Fcal$ and $q\in\Ucal$, the value $f(q)$ is the \emph{observation} returned by query $q$ in state $f$.
\end{definition}

An experiment records responses to a finite set of queries. We collect these readings into an observation vector.

\begin{definition}[Observation]\label{def:neural-measurement}
  A \emph{query set} is a nonempty finite subset $Q\subseteq\Ucal$. The \emph{observation vector} of a state $f\in\Fcal$ on $Q$ is
  \begin{equation}
    f(Q)\coloneqq\bigl(f(q)\bigr)_{q\in Q}\in\Rbb^Q.
  \end{equation}
  Here $\Rbb^Q$ denotes real-valued vectors indexed by $Q$.
\end{definition}

To compare two profiles, we ask how far their responses can differ over the queries being considered. Taking the largest difference gives a common distance for full states and finite observation vectors.

\begin{definition}[Ontic Distance]\label{def:ontic-observation-distances}
  For a nonempty index set $I$ and bounded real-valued functions $u,v$ whose domains contain $I$, define their distance on $I$ by
  \begin{equation}\label{eq:paper-distance}
    d_I(u,v)\coloneqq\sup_{i\in I}|u(i)-v(i)|.
  \end{equation}
  This is the supremum metric between their restrictions to $I$. Vectors are functions on their index set; for finite $I$, the supremum is a maximum.

  The \emph{ontic distance} is $d_{\Ucal}(f,g)$; the \emph{observation distance} on $Q$ is $d_Q(f(Q),g(Q))$. The latter is a metric on observation vectors and a pseudometric on states, since distinct states can yield identical observations on $Q$.
\end{definition}

\Cref{fig:ontic-state-observation} shows how a state and query set determine an observation vector. The noisy extension treats the profile as a conditional mean response and measured activity as a noisy observation (\cref{sec:noisy-recovery}).

\subsection{Query Ambiguity}

A complete response profile defines a state, but how much of it do our observations determine? Hume's question about inference beyond present observation brings this gap into focus~\citep[sec.~IV]{hume1902enquiries}. Within a specified model class, optimal recovery makes the uncertainty quantitative~\citep{MicchelliRivlin1977,Binev2024OptimalLearning}. We ask how different two full states can be when their recorded observations agree within a given tolerance.

\begin{definition}[Query Ambiguity]\label{def:observation-ambiguity}
  Fix a query set $Q\subseteq\Ucal$. For a tolerance $r\ge0$, define the \emph{query ambiguity} on $Q$ by
  \begin{equation}\label{eq:paper-ambiguity}
    \omega_Q(r)\coloneqq\sup_{\substack{f,g\in\Fcal\\d_Q(f(Q),g(Q))\le r}}d_{\Ucal}(f,g).
  \end{equation}
\end{definition}

At $r=0$, ambiguity measures the state differences compatible with identical observations on $Q$. Positive tolerances allow for prediction error. An \textit{observation decoder} recovers a state from the observation vector.

\begin{definition}[Observation Decoder]\label{def:observation-decoder}
  Fix a query set $Q\subseteq\Ucal$. Define the \textit{observation decoder class} by
  \begin{equation}
    \Acal_Q\coloneqq\{a_Q:\Rbb^Q\to\Fcal\}.
  \end{equation}
  For an \textit{observation decoder} $a_Q\in\Acal_Q$, its \emph{recovery error} is
  \begin{equation}
    \Ecal_Q(a_Q)\coloneqq\sup_{f\in\Fcal}d_{\Ucal}\bigl(a_Q(f(Q)),f\bigr).
  \end{equation}
  The \emph{optimal recovery error} from observations on $Q$ is
  \begin{equation}
    \Ecal_Q^\star\coloneqq\inf_{a_Q\in\Acal_Q}\Ecal_Q(a_Q).
  \end{equation}
\end{definition}

Additional queries can resolve distinctions left open by the current observations. We describe this accumulation of information by a nested query sequence.

\begin{definition}[Query Sequence]\label{def:observation-sequence}
  A \emph{query sequence} is a family $(Q_p)_{p\ge1}$ of query sets with $Q_p\subseteq Q_{p+1}\subseteq\Ucal$ for all $p\ge1$. The stage-$p$ observation vector is $f(Q_p)$.
\end{definition}

\begin{definition}[State Separation]\label{def:state-separation}
  A query sequence $(Q_p)_{p\ge1}$ is said to \emph{separate states in $\Fcal$} if
  \begin{equation}
    \forall f\ne g\in\Fcal,\quad \exists p\ge1:\quad f(Q_p)\ne g(Q_p).
  \end{equation}
\end{definition}

Separating queries progressively resolve observation ambiguity (\cref{fig:asymptotic-recovery-relations}). Compactness makes this resolution stable under vanishing tolerances.

\begin{lemma}[Vanishing Ambiguity]\label{lem:vanishing-ambiguity}
  Let $(\Fcal,d_{\Ucal})$ be nonempty and compact, and let $(P_m)_{m\ge1}$ be a query sequence separating states. For nonnegative tolerances $r_m\to0$,
  \begin{equation}\label{eq:paper-stable-ambiguity}
    \omega_{P_m}(r_m)\to0.
  \end{equation}
\end{lemma}

\subsection{Prediction and Recovery}

What must a learned representation preserve for us to recover the state it describes? Hertz's account of representation centers on the correspondence between consequences~\citep[Introduction]{hertz1899mechanics}. Here we test that correspondence by predicting further responses from encoded observations. The preceding lemma explains why the targets matter: agreement on sufficiently rich readings constrains the difference between full states. Prediction can thereby test whether a representation retains the distinctions defining the state.

\begin{definition}[Prediction and Recovery]\label{def:prediction-recovery}
  The learning pipeline uses two disjoint query sets: $Q$ indexes the input observations $f(Q)$, and $P$ indexes the target observations $f(P)$. An \textit{encoder} $e:\Rbb^Q\to\Zcal$ maps $f(Q)$ to the learned state $e(f(Q))$ in a nonempty space $\Zcal$. From this learned state, a \textit{predictor} $b:\Zcal\to\Rbb^P$ estimates $f(P)$, while a \textit{state decoder} $a:\Zcal\to\Fcal$ reconstructs the ontic state $f$. Their errors are
  \begin{equation}\label{eq:paper-errors}
    \begin{aligned}
      \Lcal_{Q,P}(b;e)&\coloneqq\sup_{f\in\Fcal}d_P\bigl((b\circ e)(f(Q)),f(P)\bigr),\\
      \Rcal_Q(a;e)&\coloneqq\sup_{f\in\Fcal}d_{\Ucal}\bigl((a\circ e)(f(Q)),f\bigr).
    \end{aligned}
  \end{equation}
  Let $\Bcal_P$ and $\Acal$ be the classes of all maps of these respective types. For a fixed \textit{encoder}, define
  \begin{equation}\label{eq:paper-optimal-errors}
    \Lcal_{Q,P}^\star(e)\coloneqq\inf_{b\in\Bcal_P}\Lcal_{Q,P}(b;e),\qquad \Rcal_Q^\star(e)\coloneqq\inf_{a\in\Acal}\Rcal_Q(a;e).
  \end{equation}
\end{definition}
Both optimizations hold the \textit{encoder} fixed; $\Rcal_Q^\star(e)$ measures the recovery accuracy available from its learned state.

\begin{corollary}[Input Ambiguity]\label{cor:input-ambiguity}
  For a nonempty state family, query set $Q$, and \textit{encoder} $e$,
  \begin{equation}
    \Rcal_Q^\star(e)\ge\Ecal_Q^\star\ge\frac12\omega_Q(0).
  \end{equation}
\end{corollary}

Input queries supply distinctions to the \textit{encoder}; targets test their preservation in its learned state.

\subsection{Low-Dimensional Geometry}

\begin{wrapfigure}{r}{.55\textwidth}
  \centering
  \noindent\makebox[\linewidth][c]{
\makebox[214.27057pt][c]{\raisebox{-20.00000pt}[84.84212pt][0.00000pt]{\includegraphics[width=254.27057pt,height=124.84212pt]{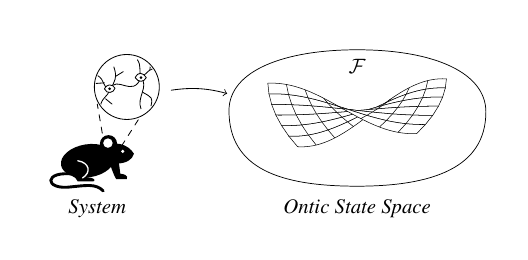}}}
}
  \caption{From population activity to state geometry. The neural inset motivates a low-dimensional model for the \textit{ontic state space} $\Fcal$.}
  \label{fig:ontic-manifold-covering}
\end{wrapfigure}

A neural recording contains many coordinates, yet activity within a task can vary through a few latent variables~\citep{Gallego2017Neuralmanifolds,Jazayeri2021Dimensions}. This organization can constrain which patterns are learned~\citep{Sadtler2014Neuralconstraints} and take forms such as the toroidal geometry of grid-cell activity~\citep{Gardner2022Toroidal}. The resulting state family may curve through a much larger linear space~\citep{DeChaudhuri2023}.

Recovering a state within such a family requires responses that distinguish its members. Classical population decoding illustrates how neural responses carry this information: movement direction can be read from motor-cortical ensembles~\citep{Georgopoulos1986Population}, and position reconstructed from hippocampal place-cell activity~\citep{Zhang1998Reconstruction}. Models of population codes explain how neuronal tuning and background activity affect readout accuracy~\citep{Seung1993Population}.

These observations lead to two quantitative requirements (\cref{fig:ontic-manifold-covering}). The state family should admit a controlled number of representatives at each resolution, and queries outside the inputs should distinguish its members. Covering numbers express the first requirement; observability expresses the second.

\begin{assumption}[Low-Dimensional Observability]\label{ass:low-dimensional-manifold}
  For a query set $Q$ with $\Ucal\setminus Q\ne\varnothing$, there are constants $C\ge1$, $d>0$, and $\alpha\in(0,1]$ such that
  \begin{equation}\label{eq:state-covering-growth}
    \Ncal_{\Fcal}(\varepsilon)\le C\varepsilon^{-d}, \qquad 0<\varepsilon\le1,
  \end{equation}
  where $\Ncal_{\Fcal}(\varepsilon)$ is the covering number: the fewest closed $d_{\Ucal}$-balls of radius $\varepsilon$, centered in $\Fcal$, that cover $\Fcal$. The queries outside $Q$ distinguish states quantitatively:
  \begin{equation}\label{eq:target-observability}
    \sup_{p\in\Ucal\setminus Q}|f(p)-g(p)|\ge\alpha d_{\Ucal}(f,g),\qquad f,g\in\Fcal.
  \end{equation}
\end{assumption}


\section{Main Results}\label{sec:paper-recovery}

Prediction tests the distinctions retained by a learned state. If an \textit{encoder} maps two states to the same learned state, every \textit{predictor} must return the same target observations for both. Their target difference therefore imposes an unavoidable prediction error.

\begin{lemma}[Prediction within Fibers]\label{lem:optimal-prediction-error}
  Assume $(\Fcal,d_{\Ucal})$ is nonempty and bounded. Fix a query set $Q\subseteq\Ucal$, a nonempty finite target query set $P\subseteq\Ucal\setminus Q$, and an \textit{encoder} $e:\Rbb^Q\to\Zcal$. The optimal prediction error is attained and equals
  \begin{equation}\label{eq:optimal-prediction-fibers}
    \Lcal_{Q,P}^\star(e)=\frac12\sup_{\substack{f,g\in\Fcal\\e(f(Q))=e(g(Q))}}d_P(f(P),g(P)).
  \end{equation}
\end{lemma}

In the one-coordinate example, two states can have identical target readings while differing elsewhere. A \textit{predictor} can be exact on those targets even if its \textit{encoder} merges the states. Recovery depends on how much full-state difference those same readings permit.

\subsection{General Case}

Target ambiguity controls how prediction accuracy translates into state recovery.

\begin{theorem}[General Prediction-Recovery Bound]\label{thm:paper-general-bounds}
  Let $(\Fcal,d_{\Ucal})$ be nonempty and bounded. For disjoint query sets $Q,P$ and any \textit{encoder} $e:\Rbb^Q\to\Zcal$,
  \begin{equation}\label{eq:paper-general-bounds}
    \mathmakebox[0.28\linewidth][r]{\Lcal_{Q,P}^\star(e)}\le\boxed{\Rcal_Q^\star(e)}\le\mathmakebox[0.28\linewidth][l]{\omega_P\bigl(2\Lcal_{Q,P}^\star(e)\bigr).}
  \end{equation}
  The upper bound also holds with $\Lcal_{Q,P}(b;e)$ for any \textit{predictor} $b$ with finite error.
\end{theorem}

When the state space is compact, a growing set of separating targets resolves the remaining ambiguity. Prediction and recovery then converge together.

\begin{theorem}[General Prediction-Recovery Equivalence]\label{thm:paper-general-equivalence}
  Let $(\Fcal,d_{\Ucal})$ be nonempty and compact, and let $(P_m)_{m\ge1}$ be a nested query sequence separating states. For each $m$, choose a query set $Q_m\subseteq\Ucal\setminus P_m$ and an \textit{encoder} $e_m:\Rbb^{Q_m}\to\Zcal_m$. Then
  \begin{equation}\label{eq:paper-general-equivalence}
    \Lcal_{Q_m,P_m}^\star(e_m)\to0\quad\Longleftrightarrow\quad\Rcal_{Q_m}^\star(e_m)\to0.
  \end{equation}
\end{theorem}

For fixed input queries and a fixed \textit{encoder}, vanishing prediction error along such a separating sequence leaves no distinct ontic states with the same learned state: an exact \textit{state decoder} exists (\cref{cor:general-state-recovery}). If the map to a Euclidean learned state is also continuous, compactness makes it a homeomorphism onto its image (\cref{cor:manifold-recovery}). Under these conditions, preserving the distinctions that define states also preserves the topology of the state space.

\subsection{Low-Dimensional Case}

\Cref{ass:low-dimensional-manifold} combines a bound on the number of representative states with the ability to distinguish them through queries outside $Q$. Together, these conditions give a finite target budget.

\begin{theorem}[Low-Dimensional Prediction-Recovery Bound]\label{thm:paper-low-bounds}
  Let $\Fcal$ be nonempty and let $(\Fcal,Q)$ satisfy~\cref{ass:low-dimensional-manifold} with constants $C,d,\alpha$. For every integer $m\ge C^2/2$, there is a target set $P\subseteq\Ucal\setminus Q$, $1\le|P|\le m$, such that every \textit{encoder} $e$ satisfies
  \begin{equation}\label{eq:paper-low-bounds}
    \Lcal_{Q,P}^\star(e)\le\boxed{\Rcal_Q^\star(e)}\le\frac{2}{\alpha}\Lcal_{Q,P}^\star(e)+\frac{2+\alpha}{\alpha}\left(\frac{C^2}{2m}\right)^{\frac{1}{2d}}.
  \end{equation}
  The upper bound also holds with $\Lcal_{Q,P}(b;e)$ for any \textit{predictor} $b$ with finite error.
\end{theorem}

The recovery allowance vanishes as the target budget grows, yielding convergence and rate equivalences.

\begin{theorem}[Low-Dimensional Prediction-Recovery Equivalence]\label{thm:paper-low-equivalence}
  Let $\Fcal$ be nonempty. Assume $(\Fcal,Q_m)$ satisfies~\cref{ass:low-dimensional-manifold} with common $C,d,\alpha$ for every integer $m\ge C^2/2$, and choose $P_m$ by~\cref{thm:paper-low-bounds}. Then every \textit{encoder} sequence $(e_m)$ satisfies
  \begin{equation}
    \Lcal_{Q_m,P_m}^\star(e_m)\to0\quad\Longleftrightarrow\quad\Rcal_{Q_m}^\star(e_m)\to0.
  \end{equation}
  Moreover,
  \begin{equation}\label{eq:paper-low-rate-equivalence}
    \Lcal_{Q_m,P_m}^\star(e_m)=O(m^{-1/(2d)})\quad\Longleftrightarrow\quad\Rcal_{Q_m}^\star(e_m)=O(m^{-1/(2d)}).
  \end{equation}
\end{theorem}

\subsection{Sharpness on a Fixed State Family}

We next examine whether the budget rate can be improved when one target set serves every \textit{encoder}. For $\alpha=1$, we measure the additive recovery cost beyond $2\Lcal^\star_{Q,P}(e)$ on a fixed state family. Here $d$ is the covering exponent; the family has topological dimension one.

\begin{definition}[Uniform Transfer Residual]\label{def:uniform-transfer-residual}
  Let $\mathcal E_{\mathrm{Lip}}$ be the globally Lipschitz maps $e:\Rbb^Q\to\Rbb^2$. For an integer $m\ge1$, the \emph{uniform transfer residual} is
  \begin{equation}\label{eq:paper-transfer-residual}
    \mathcal T_m(\Fcal,Q)\coloneqq\inf_{\substack{P\subseteq\Ucal\setminus Q\\1\le|P|\le m}}\sup_{e\in\mathcal E_{\mathrm{Lip}}}\bigl[\Rcal_Q^\star(e)-2\Lcal_{Q,P}^\star(e)\bigr].
  \end{equation}
\end{definition}
The target set is selected before the \textit{encoder}. The family below has identifying inputs and even a single query outside $Q$ whose exact response identifies the state (\cref{eq:fixed-family-exact-sampling}). An \textit{encoder} can nevertheless merge distinct states. The residual measures the additive cost of controlling their recovery by twice the prediction error, uniformly across \textit{encoders}.

\begin{theorem}[Sharp Uniform Transfer]\label{thm:paper-fixed-transfer}
  For every $d>1$, there exist a fixed countable query universe $\Ucal$, a singleton input set $Q$, and a compact state family $\Fcal$ homeomorphic to $[0,1]$, of diameter $1$. The pair $(\Fcal,Q)$ satisfies~\cref{ass:low-dimensional-manifold} with $C=C_d$ and $\alpha=1$, and its input observation map is injective. For every integer $m\ge C_d^2/2$,
  \begin{equation}\label{eq:paper-fixed-transfer}
    \mathmakebox[0.28\linewidth][r]{\frac16m^{-1/(2d)}}\le\boxed{\mathcal T_m(\Fcal,Q)}\le\mathmakebox[0.28\linewidth][l]{3\left(\frac{C_d^2}{2m}\right)^{1/(2d)}.}
  \end{equation}
  The supremum in~\cref{eq:paper-transfer-residual} has the same value when restricted to $1$-Lipschitz \textit{encoders}.
\end{theorem}

The order of selection matters as well. On this same family, a target selected for one fixed $1$-Lipschitz \textit{encoder} can achieve $\Rcal_Q^\star(e)\le2\Lcal_{Q,\{p\}}^\star(e)$ with a single query (\cref{cor:regularity-target-order}). This selection uses the complete state family. Requiring one target set to serve every \textit{encoder} gives the sharp budget in~\cref{thm:paper-fixed-transfer}. Classical optimal recovery studies the radius left by fixed information~\citep{Traub1983Information,Krieg2026Sampling}; our budget measures how many prediction targets suffice to control that radius uniformly across learned representations.

\subsection{From Recorded Error to Recovery}

The bounds above use prediction error over the entire state family. Training supplies finitely many observation pairs. To pass from recorded error to a recovery guarantee, we need to know both how these records cover possible observations and how predictions vary between them.

\begin{definition}[Recorded Error and Coverage]\label{def:recorded-error-coverage}
  Fix disjoint query sets $Q,P$, a pair $(e,b)$, and states $f_1,\ldots,f_n\in\Fcal$ with $n\ge1$ indexing the recorded pairs $(f_i(Q),f_i(P))$. The \emph{empirical prediction error} and \emph{observation coverage radius} are
  \begin{equation}\label{eq:empirical-record-error}
    \widehat{\Lcal}_{Q,P,n}(b;e)\coloneqq\max_{1\le i\le n}d_P\bigl((b\circ e)(f_i(Q)),f_i(P)\bigr),
  \end{equation}
  \begin{equation}\label{eq:training-coverage}
    h_n\coloneqq\sup_{f\in\Fcal}\min_{1\le i\le n}d_{Q\cup P}\bigl(f(Q\cup P),f_i(Q\cup P)\bigr).
  \end{equation}
\end{definition}

The radius measures coverage jointly in inputs and targets: every possible observation pair lies within $h_n$ of a record. Prediction stability controls how the learned map responds over this distance.

\begin{assumption}[Prediction Stability]\label{ass:uniform-prediction-stability}
  For the pair $(e,b)$, there is a nondecreasing function $\rho:[0,\infty)\to[0,\infty)$ with $\rho(0)=0$ and $\rho(t)\to0$ as $t\downarrow0$ such that, for all $f,g\in\Fcal$,
  \begin{equation}
    d_P\bigl((b\circ e)(f(Q)),(b\circ e)(g(Q))\bigr) \le\rho\bigl(d_Q(f(Q),g(Q))\bigr).
  \end{equation}
\end{assumption}

For a nearby record, the prediction changes by at most $\rho(h_n)$ and the target reading by at most $h_n$. Adding these changes to the recorded error controls prediction at the unrecorded state. A Lipschitz prediction map has $\rho(t)=Kt$.

\begin{proposition}[Recovery from Records]\label{prop:record-recovery}
  Let $(\Fcal,d_{\Ucal})$ be nonempty and compact. For a pair $(e,b)$ satisfying~\cref{ass:uniform-prediction-stability}, the prediction error obeys
  \begin{gather}
    E_n\coloneqq\widehat{\Lcal}_{Q,P,n}(b;e)+\rho(h_n)+h_n,\label{eq:paper-record-error}\\
    \mathmakebox[0.20\linewidth][r]{\widehat{\Lcal}_{Q,P,n}(b;e)}\le\boxed{\Lcal_{Q,P}(b;e)}\le\mathmakebox[0.20\linewidth][l]{E_n,}\label{eq:uniform-prediction-error-bounds}
  \end{gather}
  and the recovery error satisfies
  \begin{equation}\label{eq:paper-record-recovery}
    \Rcal_Q^\star(e)\le\omega_P(2E_n).
  \end{equation}
  Under~\cref{ass:low-dimensional-manifold}, for targets selected by~\cref{thm:paper-low-bounds},
  \begin{equation}\label{eq:paper-record-low-recovery}
    \Rcal_Q^\star(e)\le\frac{2}{\alpha}E_n+\frac{2+\alpha}{\alpha}\left(\frac{C^2}{2m}\right)^{1/(2d)}.
  \end{equation}
\end{proposition}

The terms answer different questions. The empirical error measures fit on recorded pairs. Their coverage and the prediction modulus control what this fit implies between records. The target budget controls distinctions left unresolved even by a uniformly accurate \textit{predictor}. More training records and more target queries address these two sources of uncertainty separately. The fine-tuning results in~\cref{thm:general-finetuning-recovery,thm:low-dimensional-finetuning-recovery} include model-class error and optimization tolerance; \cref{sec:noisy-recovery} also treats measurement perturbations.

\section{Experiment}\label{sec:experiments}

The experiments examine prediction and recovery with two kinds of access to the state. On \textit{Sphere}, \textit{Torus}, and \textit{General} families, the known response law makes the complete profile available for evaluation. Random target sampling tests separation and budget dependence; \cref{thm:paper-fixed-transfer} supplies the sharpness result analytically. Neural recordings then let us examine covering geometry, held-out prediction, and recovery of a fixed model's hidden state as an ontic-state surrogate.

\subsection{Synthetic Experiment}\label{sec:synthetic-experiments}

\begin{wrapfigure}{l}{.55\textwidth}
  \centering
  \begin{adjustbox}{max width=\linewidth}
\makebox[218.89830pt][c]{\raisebox{-20.00000pt}[87.32524pt][0.00000pt]{\includegraphics[width=258.89830pt,height=127.32524pt]{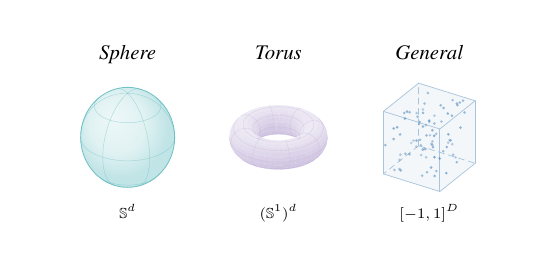}}}

  \end{adjustbox}
  \caption{Synthetic state families. Representative low-dimensional views of the sphere, product of circles and independent box.}
  \label{fig:synthetic-family-shapes}
\end{wrapfigure}

We use three synthetic state families (\cref{fig:synthetic-family-shapes}): \textit{Sphere} is uniform on a unit sphere; \textit{Torus} consists of independent circle factors; \textit{General} has independent uniform coordinates in a high-dimensional box. A common observation law embeds them into the same query universe. We jointly train an \textit{encoder} and \textit{predictor} on input--target observations, then recover the state through a deterministic ground-truth (GT) map, following~\cref{thm:paper-general-bounds}. Here the coordinate structure determines the random-sampling budget; the experiment tests how target separation transfers prediction accuracy to recovery.

For a latent vector $x\in[-1,1]^r$ in one of the three families, define
\begin{equation}\label{eq:synthetic-gt-law}
  f_x(j,k)=\tanh\bigl(a_{jk}s_jx_{i(j)}+c_{jk}\bigr),
\end{equation}
where $i(j)$ selects a coordinate, $s_j\in\{-1,1\}$, and $a_{jk}>0,c_{jk}\in\Rbb$ are fixed channel coefficients.
Each family uses the same coefficient law. Repeating low-dimensional coordinates across the query universe changes the geometry of the state family while retaining this observation mechanism.

Given predicted target observations $y\in\Rbb^P$, the ideal deterministic GT map returns
\begin{equation}\label{eq:synthetic-gt-oracle}
  \widehat f(y)\in\operatorname*{arg\,min}_{g\in\Fcal}d_P\bigl(g(P),y\bigr).
\end{equation}
The known observation law converts target-error tolerances into coordinate intervals. A deterministic feasibility calculation approximates this map (\cref{sec:synthetic-details}) and supplies the complete recovered profile, including responses outside the target set.

\mypara{Observability check.}
To apply~\cref{thm:paper-general-bounds}, we need to control target ambiguity. Complete coordinate coverage gives an explicit constant before training.

\begin{proposition}[Coordinate Observability]\label{prop:synthetic-observability}
  For the family in~\cref{eq:synthetic-gt-law}, suppose $P$ observes every latent coordinate. Let $\ell_i$ be the largest $a_{jk}\operatorname{sech}^2(a_{jk}+|c_{jk}|)$ among target channels observing coordinate $i$, and let $A_i$ be the largest $a_{jk}$ among all channels observing it. Then, for all $f,g\in\Fcal$,
  \begin{equation}\label{eq:synthetic-main-observability}
    \alpha_P=\min_i\frac{\ell_i}{A_i}>0,\qquad d_P(f,g)\ge\alpha_Pd_{\Ucal}(f,g).
  \end{equation}
\end{proposition}

This inequality gives $\omega_P(t)\le t/\alpha_P$, so~\cref{thm:paper-general-bounds} yields $\Rcal_Q^\star(e)\le2\Lcal_{Q,P}(b;e)/\alpha_P$. The observation law and target set determine this global constant; predictive training determines the target error. Together they control differences in the complete profile. The budget calculation and covering constants appear in~\cref{tab:synthetic-budget-diagnostic,sec:synthetic-details}.

\mypara{Error transfer.}
With target separation established, \cref{fig:synthetic-gt-main} follows the prediction-to-recovery relation. The left panel compares finite-test errors with the exact-map benchmark. The middle panel gives the sufficient separation budget under the coordinate-sampling law. The right panel extends the recovery comparison to the full known family through the conditional numerical audit below.

\begin{wraptable}{r}{.55\textwidth}
  \centering
  \caption{Conditional whole-domain bounds for the ideal GT map: mean $\pm$ sample standard deviation across paired runs. Arithmetic conditions: \cref{sec:synthetic-uniform}.}
  \label{tab:synthetic-uniform}
  \small
  \setlength{\tabcolsep}{2pt}
  \begin{tabular*}{\linewidth}{@{\extracolsep{\fill}}lrrr@{}}
    \toprule
    Quantity & \textit{Sphere} & \textit{Torus} & \textit{General} \\
    \midrule
    $\alpha_P$ & $.2308\pm.0050$ & $.1831\pm.0160$ & $.0800\pm.0000$ \\
    $\underline L$ & $.1020\pm.0181$ & $.0043\pm.0008$ & $.0078\pm.0019$ \\
    $\overline L$ & $.1031\pm.0183$ & $.0059\pm.0008$ & $.0093\pm.0023$ \\
    $2\overline L/\alpha_P$ & $.8952\pm.1719$ & $.0644\pm.0107$ & $.2333\pm.0570$ \\
    \bottomrule
  \end{tabular*}
\end{wraptable}

\mypara{Whole-domain audit.}
A test set samples prediction error; the theoretical quantity is a supremum over the state family. Where~\cref{prop:record-recovery} controls this supremum through observation coverage, the known coordinatewise law permits a scalar-domain audit. For each target, its error is a scalar function of the observed latent coordinate. A fixed grid and a sensitivity bound computed from the saved weights enclose its error between grid points. Under the arithmetic-error assumptions in~\cref{sec:synthetic-uniform}, the grid audit gives a conditional upper bound throughout the known family. Applying the ideal GT map yields~\cref{tab:synthetic-uniform}.

Every paired run at the audited separating budgets yields an ideal-map bound below the full-profile diameter (\cref{tab:synthetic-uniform}). The bound is evaluated for each trained \textit{predictor} using the protocol in~\cref{sec:synthetic-uniform}.

For joint queries, \cref{sec:synthetic-mixed} derives a global prediction-to-recovery certificate from calibration records at known latent coordinates.

\begin{figure*}[t]
  \centering
  \begin{minipage}[t]{0.48\textwidth}
    \vspace{0pt}
    \centering
\makebox[173.30862pt][c]{\raisebox{-20.00000pt}[78.39769pt][0.00000pt]{\includegraphics[width=213.30862pt,height=118.39769pt]{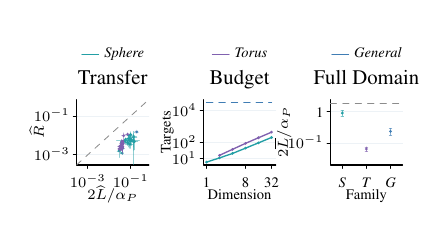}}}

    \par
    \caption{Prediction-to-recovery transfer. \emph{Left:} test errors and exact-map benchmark. \emph{Middle:} budgets for $95\%$ separation. \emph{Right:} conditional whole-domain bounds. Dashed lines: independent-family budget and profile diameter. Bars: repetition standard deviations. See~\cref{tab:synthetic-budget-diagnostic,tab:synthetic-uniform}.}
    \label{fig:synthetic-gt-main}
  \end{minipage}\hfill
  \begin{minipage}[t]{0.48\textwidth}
    \vspace{0pt}
    \centering
\makebox[184.36829pt][c]{\raisebox{-20.00000pt}[78.39769pt][0.00000pt]{\includegraphics[width=224.36829pt,height=118.39769pt]{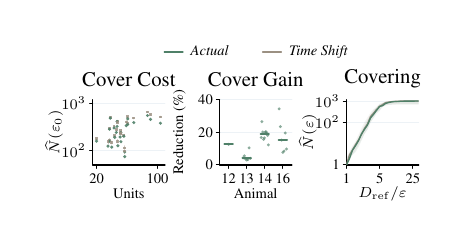}}}

    \par
    \caption{HC-3 covering geometry. \emph{Left:} fixed-radius cover sizes; bars span shift controls. \emph{Middle:} session reductions relative to median shift controls, grouped by animal; bars mark animal medians. \emph{Right:} median covering curves with session interquartile bands. See~\cref{sec:hc3-covering-details}.}
    \label{fig:hc3-geometry}
  \end{minipage}
\end{figure*}

\subsection{Real Experiment}\label{sec:real-experiments}

In neural recordings, we examine covering complexity, held-out target prediction, and recovery of a fixed reference model's hidden coordinates. These coordinates define the surrogate state whose preservation we evaluate.

\par\ifdim\dimexpr\pagegoal-\pagetotal\relax<16\baselineskip\newpage\fi
\begin{wraptable}{l}{.55\textwidth}
  \centering
  \caption{Diagnostic-block greedy covers at the fixed radius in~\cref{tab:hc3-cover-protocol}. Variance-retention ranks use reference-derived directions.}
  \label{tab:hc3-geometry-summary}
  \small
  \setlength{\tabcolsep}{4pt}
  \renewcommand{\arraystretch}{1.12}
  \begin{tabular*}{\linewidth}{@{\extracolsep{\fill}}lrr@{}}
    \toprule
    Measure & Median & Range \\
    \midrule
    Greedy-cover size & 243 & 75--554 \\
    \quad Fraction of observations & 31.7\% & 7.3\%--71.2\% \\
    \quad Reduction vs. shift & 15.2\% & 3.0\%--34.3\% \\
    \addlinespace[4pt]
    Rank for 90\% variance & 33 & 14--63 \\
    Rank for 95\% variance & 37 & 15--73 \\
    \bottomrule
  \end{tabular*}
\end{wraptable}

\mypara{Observation geometry.}
The HC-3 hippocampal--entorhinal recordings~\citep{mizuseki2014neurosharing} span animals and behavioral contexts. Chronological blocks separate preprocessing, reference geometry, diagnostic geometry, and prediction evaluation (\cref{sec:hc3-covering-details}).

We cover the diagnostic block with observed representatives in the maximum-coordinate metric (\cref{fig:hc3-geometry}). At a fixed resolution, every session needs a smaller greedy cover than its median time-shift control, with a median reduction of $15.2\%$. These controls preserve each unit's marginal distribution and circular autocorrelation while changing alignment between units. The comparison holds within every selected animal and across shorter and longer counting windows (\cref{tab:hc3-geometry-summary,tab:hc3-cover-animal,tab:hc3-cover-window}). Covering curves measure the finite-record analog of the approximation complexity in~\cref{eq:state-covering-growth}. Variance-retention ranks measure concentration along reference-derived covariance directions.

\begin{wraptable}{r}{.55\textwidth}
  \centering
  \caption{Cover-size reduction within each selected animal. Values are session medians relative to the matched time-shift median at the fixed radius. Every selected session contributes.}
  \label{tab:hc3-cover-animal}
  \small
  \begin{tabular*}{\linewidth}{@{\extracolsep{\fill}}lrrrr@{}}
    \toprule
    Animal & Sessions & $.5$\,s & $1$\,s & $2$\,s \\
    \midrule
    ec012 & 1 & 12.3\% & 12.6\% & 15.0\% \\
    ec013 & 7 & 4.7\% & 4.1\% & 4.9\% \\
    ec014 & 10 & 14.0\% & 18.9\% & 19.1\% \\
    ec016 & 7 & 15.3\% & 15.2\% & 13.2\% \\
    \bottomrule
  \end{tabular*}
\end{wraptable}

Data and controls use matched radii and sample sizes (\cref{sec:hc3-covering-details}). \Cref{tab:hc3-cover-animal} compares counting windows within each animal.

\mypara{Target prediction.}
We next ask whether a learned representation predicts a disjoint target set. An \textit{encoder--predictor} learns from the same sessions' training blocks and is evaluated on later held-out observations. At the larger target budgets, prediction improves on a constant \textit{predictor} derived from the training block (\cref{fig:hc3-local-recovery}, left).

\mypara{Reference state recovery.}
The trained \textit{encoder} above supplies the reference model: its full hidden-coordinate vector defines each observation's surrogate state. A smaller model learns the same prediction task independently, using the same observations and chronological split. We then recover the reference state by local interpolation in the smaller model's learned space: reference states at nearby training inputs are averaged with distance-based weights. The reference states enter the recovery map after predictive training.

\par\ifdim\dimexpr\pagegoal-\pagetotal\relax<17\baselineskip\newpage\fi
\begin{wrapfigure}{l}{.55\textwidth}
  \centering
  \begin{adjustbox}{max width=\linewidth}
\makebox[211.35010pt][c]{\raisebox{-20.00000pt}[99.62103pt][0.00000pt]{\includegraphics[width=251.35010pt,height=139.62103pt]{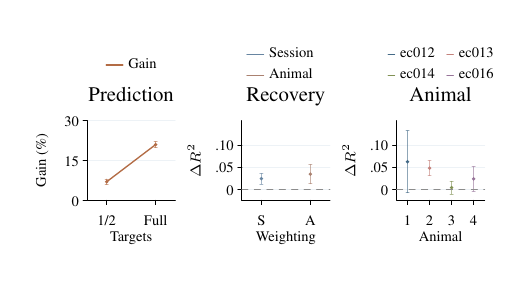}}}

  \end{adjustbox}
  \caption{Prediction and reference-state recovery. \emph{Left:} prediction gain. \emph{Middle:} recovery gain with session (S) or animal (A) weighting. \emph{Right:} animal gains. Repetition means and standard deviations; see~\cref{tab:hc3-prediction,tab:local-summary}.}
  \label{fig:hc3-local-recovery}
\end{wrapfigure}

The recovery score measures reconstruction of the surrogate's full coordinates on held-out observations. Predictive training improves reconstruction of these reference coordinates relative to the same small \textit{encoder} at initialization, with a session-weighted gain of $.0239\pm.0118$ in $R^2$ (\cref{fig:hc3-local-recovery,tab:local-summary}). The mean gain remains positive within each animal and each paired seed group, and after leaving out any animal. The protocol and cohort summaries appear in~\cref{sec:hc3-local-details}.

\par\raggedbottom\WFclear\vspace*{8pt}
\section{Related Work}

\mypara{Prediction and state learning.}
Predictive state models describe dynamics through action-conditioned future observations~\citep{LittmanSuttonSingh2001}. Auxiliary prediction and context learning link predictive objectives to downstream performance~\citep{Lee2021ProvableSSL,Zhai2025Contextures}. \citet{Zhang2025WorldModels} recover generating latent variables up to coordinate negations and permutations as tasks increase, under a low-degree task distribution and preference for low-degree realizations. Our target is the complete query-response profile; the fixed-family bound quantifies the cost of one target set serving all \textit{encoders}.

\mypara{Identifiability.}
Identifiability depends on the learning setup~\citep{Locatello2019}. Auxiliary variables~\citep{Khemakhem2020Identifiable}, contrastive objectives with generative assumptions~\citep{Zimmermann2021Contrastive}, and multiple modalities~\citep{Daunhawer2023Multimodal} can identify factors up to transformations. Our recovery target is specified by query responses before choosing latent coordinates.

\mypara{Optimal recovery and sampling.}
Optimal recovery studies uncertainty under partial observations~\citep{MicchelliRivlin1977}. Here, states sharing an \textit{encoder} output form a set whose diameter limits recovery; prediction constrains that set through its targets. Structured model classes also yield recovery and sampling guarantees~\citep{Binev2024OptimalLearning,Krieg2026Sampling}.

\mypara{Geometry and stable approximation.}
Embeddings~\citep{Sauer1991Embedology} and random manifold projections~\citep{Baraniuk2009Manifolds} relate geometry to measurement requirements. Stable manifold widths link approximation to metric entropy~\citep{Cohen2022StableApproximation}. Here, low-dimensional covering controls the distinctions to resolve; observability supplies the queries.

\mypara{Neural geometry and computation.}
Population geometry relates to computation~\citep{Chung2021Geometry,Jazayeri2021Dimensions}, cross-task generalization~\citep{Wakhloo2026OptimalCoding}, and cross-animal decoding~\citep{Safaie2023Preserved}. Intrinsic-dimension estimation characterizes the degrees of freedom in neural activity~\citep{Fadanni2026Dimensions}. These studies motivate our geometric assumptions, which connect covering complexity to query budgets.

\mypara{Neural state learning.}
LFADS infers population dynamics~\citep{Pandarinath2018Inferringsingle}; the Neural Latents Benchmark evaluates held-out activity and behavioral readouts~\citep{Pei2021NLB}. Multi-context, multi-task, and foundation models learn across recordings~\citep{Ye2023NDT2,Azabou2025POYO,Wang2025Foundation}. CEBRA uses behavioral or temporal context~\citep{Schneider2023CEBRA}, DPAD prioritizes behaviorally relevant dynamics~\citep{Sani2024DPAD}, and MARBLE represents local manifold dynamics~\citep{Gosztolai2025MARBLE}. These objectives reward different distinctions; we ask when preserving them suffices to recover the complete response profile.

\begin{samepage}
  \section{Conclusion}

  We studied when learning to predict a system enables recovery of its ontic state. Defining that state as a complete response profile lets us relate prediction to the distinctions a learned representation preserves. Query ambiguity makes this connection quantitative, yielding recovery guarantees in general and low-dimensional settings. In the latter, we establish query-budget bounds whose order is sharp on a fixed state family when one target set serves every \textit{encoder}. These results give a precise sense in which predicting observations can become a way of knowing the state that produces them. We hope that our work offers theoretical support for using deep learning models as computational surrogates for complex systems, including those studied in neuroscience.

\end{samepage}

\clearpage
\makeatletter
\begin{center}
  \LARGE\scshape
  \@title\\(Appendix)
\end{center}
\makeatother
\vspace{6pt}
\appendix
\raggedbottom
\addtocontents{toc}{\protect\setcounter{tocdepth}{3}}

\noindent\textbf{\large Contents of the Appendix}\par
\vspace{0.5em}
\begingroup
  \setlength{\parskip}{0pt}
  \hypersetup{linkcolor=blue}
  \makeatletter
  \@starttoc{toc}
  \makeatother
\endgroup
\clearpage

\section{State, Query, and Observation}\label{sec:appendix-observation}

We begin with observation ambiguity, then derive the prediction--recovery bounds and their sharpness. The training and noisy-observation results extend this chain to finite records and perturbed measurements. Experimental details follow the main-text order: synthetic recovery, joint-query calibration, and neural observation geometry and surrogate-state recovery.

\subsection{Distance and Restriction}

Distances contract when we restrict the query set. For query sets $Q\subseteq Q'\subseteq\Ucal$ and states $f,g\in\Fcal$,
\begin{equation}\label{eq:observation-contraction}
  d_Q(f(Q),g(Q))\le d_{Q'}(f(Q'),g(Q'))\le d_{\Ucal}(f,g).
\end{equation}
Each inequality follows by taking the supremum over a subset of the same observation coordinates. In particular, every observation map $f\mapsto f(Q)$ is $1$-Lipschitz and hence continuous.

\begin{lemma}[Ontic Metric]\label{lem:ontic-metric-space}
  The function $d_{\Ucal}:\Fcal\times\Fcal\to[0,\infty)$ is a metric on $\Fcal$. Hence $(\Fcal,d_{\Ucal})$ is a metric space.
\end{lemma}

\begin{proof}
  Boundedness of $f,g\in\Fcal$ gives
  \begin{equation}
    d_{\Ucal}(f,g)\le\sup_{q\in\Ucal}|f(q)|+\sup_{q\in\Ucal}|g(q)|<\infty.
  \end{equation}
  Nonnegativity and symmetry follow from absolute value. Moreover, $d_{\Ucal}(f,g)=0$ holds exactly when $f(q)=g(q)$ at every query, which means $f=g$ as functions. Finally, for every $q\in\Ucal$,
  \begin{equation}
    |f(q)-g(q)|\le|f(q)-h(q)|+|h(q)-g(q)| \le d_{\Ucal}(f,h)+d_{\Ucal}(h,g).
  \end{equation}
  Taking the supremum over $q$ proves the triangle inequality.
\end{proof}

\subsection{A Common Readout}

States sharing the same input must receive the same output. Their target spread limits the accuracy of a common readout.

\begin{lemma}[Fiberwise Error Bound]\label{lem:fiberwise-readout}
  Let $\Fcal$ be nonempty and $(\mathcal V,d)$ a metric space. The map $T:\Fcal\to\mathcal Y$ specifies the input available to a readout, and $G:\Fcal\to\mathcal V$ specifies its target. Define the largest target diameter among states with the same input by
  \begin{equation}
    \Delta\coloneqq\sup_{\substack{f,g\in\Fcal\\T(f)=T(g)}}d(G(f),G(g)).
  \end{equation}
  Then
  \begin{equation}\label{eq:fiberwise-error-bounds}
    \frac{\Delta}{2}\le\boxed{\inf_{c:\mathcal Y\to\mathcal V}\sup_{f\in\Fcal}d(c(T(f)),G(f))}\le\Delta.
  \end{equation}
  If $(\mathcal V,d)=(\Rbb^P,d_P)$ with $P$ nonempty and finite and $G(\Fcal)$ bounded, the coordinatewise midpoint attains the lower bound $\Delta/2$.
\end{lemma}

\begin{proof}
  \emph{Choose a common output.} For each attainable input $y\in T(\Fcal)$, let $S_y\coloneqq\{f\in\Fcal:T(f)=y\}$ be its fiber. Choose $h_y\in S_y$ and define $c(y)=G(h_y)$. Assign a fixed element of $\mathcal V$ at unattainable inputs. For any $f\in S_y$,
  \begin{equation}
    d(c(T(f)),G(f))=d(G(h_y),G(f))\le\Delta.
  \end{equation}
  Taking the supremum over $f\in\Fcal$ gives the general upper bound.

  The state-valued case $G(f)=f$ is shown below for three fibers.
  \begin{center}
    \makebox[209.57780pt][c]{\raisebox{-20.00000pt}[68.52885pt][0.00000pt]{\includegraphics[width=249.57780pt,height=108.52885pt]{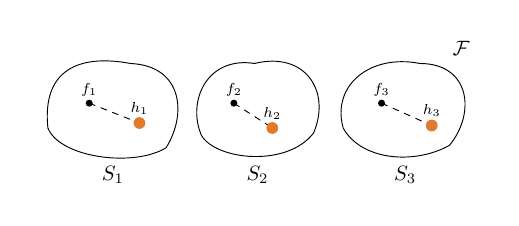}}}

  \end{center}

  \emph{Compare any common output.} If $T(f)=T(g)$, every readout $c$ gives the same output to both states. The triangle inequality gives
  \begin{equation}
    d(G(f),G(g))\le d(G(f),c(T(f)))+d(c(T(g)),G(g)) \le2\sup_{h\in\Fcal}d(c(T(h)),G(h)).
  \end{equation}
  Taking the supremum over pairs and then the infimum over readouts gives the lower bound $\Delta/2$.

  \emph{Use coordinate midpoints.} When $(\mathcal V,d)=(\Rbb^P,d_P)$ and $G(\Fcal)$ is bounded, define the coordinatewise midpoint
  \begin{equation}\label{eq:fiberwise-midpoint}
    [c^\star(y)](p)\coloneqq\frac12\Bigl(\inf_{g\in S_y}[G(g)](p) +\sup_{g\in S_y}[G(g)](p)\Bigr),\qquad p\in P.
  \end{equation}
  The endpoints are finite by boundedness. Set $c^\star(y)=0\in\Rbb^P$ at unattainable inputs. For any $f\in S_y$ and $p\in P$,
  \begin{equation}
    \bigl|[c^\star(y)](p)-[G(f)](p)\bigr| \le\frac12\sup_{g,h\in S_y}\bigl|[G(g)](p)-[G(h)](p)\bigr| \le\frac{\Delta}{2}.
  \end{equation}
  Taking the maximum over $p\in P$ and the supremum over $f\in\Fcal$ shows that $c^\star$ attains the lower bound.
\end{proof}

\begin{lemma}[Exact Decoding]\label{lem:exact-decoding}
  For nonempty $\Fcal$, the following are equivalent: $\Rcal_Q^\star(e)=0$; the map $f\mapsto e(f(Q))$ is injective; and some \textit{state decoder} satisfies $(a\circ e)(f(Q))=f$ for every $f\in\Fcal$.
\end{lemma}

\begin{proof}
  If $\Rcal_Q^\star(e)=0$, \cref{lem:fiberwise-readout} gives $d_{\Ucal}(f,g)=0$ whenever $e(f(Q))=e(g(Q))$, so the map is injective. Its inverse on attainable learned states, extended by a fixed member of $\Fcal$ elsewhere, is an exact \textit{decoder}. Conversely, an exact \textit{decoder} has zero error, so $\Rcal_Q^\star(e)=0$.
\end{proof}

\subsection{Ambiguity and Recovery}

The common-readout construction now applies to all states compatible with an observation vector. Compactness will ensure that their ambiguity vanishes as separating queries accumulate.

\begin{assumption}[Compactness]\label{ass:compactness}
  The ontic state space $(\Fcal,d_{\Ucal})$ is nonempty and compact.
\end{assumption}

\begin{lemma}[Ambiguity and Recovery]\label{lem:ambiguity-recovery}
  Assume $\Fcal$ is nonempty. For every query set $Q\subseteq\Ucal$,
  \begin{equation}
    \frac{1}{2}\omega_Q(0)\le\boxed{\Ecal_Q^\star}\le\omega_Q(0).
  \end{equation}
\end{lemma}

\begin{proof}
  By~\cref{lem:ontic-metric-space}, $(\Fcal,d_{\Ucal})$ is a metric space. Apply~\cref{lem:fiberwise-readout} with $T(f)=f(Q)$, $G(f)=f$ and $(\mathcal V,d)=(\Fcal,d_{\Ucal})$. Its fiber diameter is $\Delta=\omega_Q(0)$, so~\cref{eq:fiberwise-error-bounds} gives both bounds.
\end{proof}

\begin{proof}[Proof of~\cref{lem:vanishing-ambiguity}]
  Suppose the conclusion fails. There are $\varepsilon>0$, increasing indices $m_k$, and states $f_k,g_k\in\Fcal$ such that
  \begin{equation}
    d_{P_{m_k}}\bigl(f_k(P_{m_k}),g_k(P_{m_k})\bigr)\le r_{m_k},\qquad d_{\Ucal}(f_k,g_k)\ge\varepsilon.
  \end{equation}
  Compactness gives a common subsequence with $f_k\to f$ and $g_k\to g$. Continuity of the metric gives $d_{\Ucal}(f,g)\ge\varepsilon$.

  Fix any stage $m$. Eventually $m_k\ge m$, so nestedness and~\cref{eq:observation-contraction} give
  \begin{equation}
    d_{P_m}\bigl(f_k(P_m),g_k(P_m)\bigr)\le r_{m_k}\to0.
  \end{equation}
  Taking the limit yields $f(P_m)=g(P_m)$. This holds at every stage, so state separation forces $f=g$, contradicting $d_{\Ucal}(f,g)\ge\varepsilon$.
\end{proof}

Before introducing a learned representation, we characterize when the observations themselves suffice for recovery.

\begin{theorem}[Separation and Recovery]\label{thm:asymptotic-recovery}
  Assume $(\Fcal,d_{\Ucal})$ is nonempty and compact~(\cref{ass:compactness}) and let $(Q_p)_{p\ge1}$ be a query sequence. Then
  \begin{equation}
    \text{State separation}\quad\Longleftrightarrow\quad\Ecal_{Q_p}^\star\xrightarrow{p\to\infty}0.
  \end{equation}
\end{theorem}

\begin{proof}
  If the queries separate states, apply~\cref{lem:vanishing-ambiguity} with zero tolerance and then~\cref{lem:ambiguity-recovery}:
  \begin{equation}
    0\le\Ecal_{Q_p}^\star\le\omega_{Q_p}(0)\to0.
  \end{equation}
  Conversely, if two distinct states $f,g$ agree on every $Q_p$, then~\cref{lem:ambiguity-recovery} gives
  \begin{equation}
    \Ecal_{Q_p}^\star\ge\frac12\omega_{Q_p}(0)\ge\frac12d_{\Ucal}(f,g)>0
  \end{equation}
  at every stage. Thus vanishing recovery error requires state separation.
\end{proof}

The diagram shows candidate sets under state separation.
\begin{figure}[H]
  \centering
  \noindent\makebox[\linewidth][c]{\hspace*{2.5bp}\makebox[188.59280pt][c]{\raisebox{-20.00000pt}[107.24794pt][0.00000pt]{\includegraphics[width=228.59280pt,height=147.24794pt]{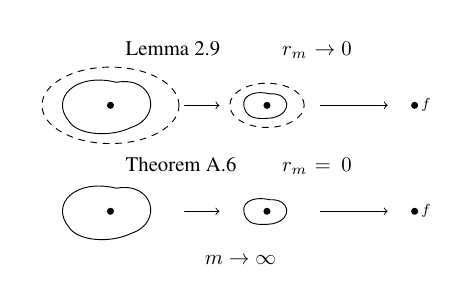}}}
\hspace*{-2.5bp}}
  \caption{Schematic candidate sets for a fixed $f$: exact observations (solid) and tolerance $r_m$ (dashed). Both have vanishing diameter under compactness and state separation, with $r_m\to0$ for the tolerance sets.}
  \label{fig:asymptotic-recovery-relations}
\end{figure}

\section{Prediction Setup}
\label{sec:appendix-prediction-setup}

We now compare prediction and recovery from the learned state $e(f(Q))$, using the error criteria in~\cref{def:prediction-recovery}. In the figures, $z_i=e(f_i(Q))$.

\begin{proof}[Proof of~\cref{cor:input-ambiguity}]
  Any \textit{state decoder} $a$ can also recover from the original observations: first apply $e$, then $a$. This gives an \textit{observation decoder} with the same error, so
  \begin{equation}
    \Ecal_Q^\star\le\Ecal_Q(a\circ e)=\Rcal_Q(a;e).
  \end{equation}
  Taking the infimum over $a$ gives $\Ecal_Q^\star\le\Rcal_Q^\star(e)$. The remaining bound $\Ecal_Q^\star\ge\omega_Q(0)/2$ follows from~\cref{lem:ambiguity-recovery}.
\end{proof}

The recovery criterion is $\Rcal_{Q_m}^\star(e_m)\to0$: for every desired accuracy, each \textit{encoder} at a sufficiently late stage admits a \textit{decoder} achieving that accuracy uniformly over states and ontic queries. The \textit{decoder} may depend on the stage.

\begin{proof}[Proof of~\cref{lem:optimal-prediction-error}]
  Apply~\cref{lem:fiberwise-readout} with $T(f)=e(f(Q))$, $G(f)=f(P)$ and $(\mathcal V,d)=(\Rbb^P,d_P)$. Since $\Fcal$ is bounded and restriction contracts distances, $G(\Fcal)$ is bounded. The coordinatewise midpoint therefore attains the common lower bound, giving~\cref{eq:optimal-prediction-fibers}.
\end{proof}

\section{Prediction and Recovery}\label{sec:prediction-state-recovery}

We prove~\cref{thm:paper-general-bounds,thm:paper-general-equivalence} first, followed by~\cref{thm:paper-low-bounds,thm:paper-low-equivalence}. In each case, fixing the \textit{encoder} gives an exact-recovery corollary.

\subsection{General Case}

The \textit{predictor} assigns a common target vector to states merged by the \textit{encoder}. Their target spread controls their full-state spread through query ambiguity.

\begin{proof}[Proof of~\cref{thm:paper-general-bounds}]
  We apply~\cref{lem:fiberwise-readout} to ontic states mapped to the same learned state, taking $T(f)=e(f(Q))$, $G(f)=f$ and $(\mathcal V,d)=(\Fcal,d_{\Ucal})$.

  \emph{Lower bound.} Restricting a decoded state to the target queries gives a \textit{predictor}, as illustrated below.
  \begin{center}
    \noindent\makebox[\linewidth][c]{
      \makebox[202.41420pt][c]{\raisebox{-20.00000pt}[90.00508pt][0.00000pt]{\includegraphics[width=242.41420pt,height=130.00508pt]{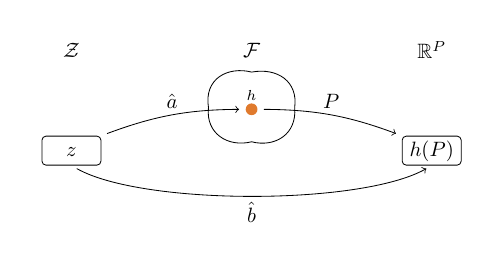}}}

    }
  \end{center}
  Target observations differ by at most the full state distance. Combining this with~\cref{lem:optimal-prediction-error} and the lower bound in~\cref{lem:fiberwise-readout} gives
  \begin{equation}
    \Lcal_{Q,P}^\star(e)=\frac12\sup_{\substack{f,g\in\Fcal\\e(f(Q))=e(g(Q))}}d_P(f(P),g(P)) \le\frac12\sup_{\substack{f,g\in\Fcal\\e(f(Q))=e(g(Q))}}d_{\Ucal}(f,g) \le\Rcal_Q^\star(e).
  \end{equation}

  \emph{Upper bound.} By~\cref{lem:optimal-prediction-error}, any pair with $e(f(Q))=e(g(Q))$ has target distance at most $2\Lcal_{Q,P}^\star(e)$. The definition of query ambiguity and~\cref{lem:fiberwise-readout} give
  \begin{equation}
    \Rcal_Q^\star(e)\le\sup_{\substack{f,g\in\Fcal\\e(f(Q))=e(g(Q))}}d_{\Ucal}(f,g) \le\omega_P\bigl(2\Lcal_{Q,P}^\star(e)\bigr).
  \end{equation}
  For a \textit{predictor} $b$ with finite error, the same bound follows from $\Lcal_{Q,P}^\star(e)\le\Lcal_{Q,P}(b;e)$ and the monotonicity of $\omega_P$.
\end{proof}

\begin{proof}[Proof of~\cref{thm:paper-general-equivalence}]
  If $\Lcal_{Q_m,P_m}^\star(e_m)\to0$, apply~\cref{lem:vanishing-ambiguity} with tolerance $2\Lcal_{Q_m,P_m}^\star(e_m)$. By~\cref{thm:paper-general-bounds},
  \begin{equation}
    0\le\Rcal_{Q_m}^\star(e_m)\le\omega_{P_m}\bigl(2\Lcal_{Q_m,P_m}^\star(e_m)\bigr)\to0.
  \end{equation}
  Conversely, $0\le\Lcal_{Q_m,P_m}^\star(e_m)\le\Rcal_{Q_m}^\star(e_m)$ gives vanishing prediction error whenever recovery error vanishes.
\end{proof}

For a fixed \textit{encoder}, vanishing prediction error gives exact state recovery.

\begin{corollary}[General Exact Recovery]\label{cor:general-state-recovery}
  Under~\cref{ass:compactness}, fix a query set $Q$ and an \textit{encoder} $e:\Rbb^Q\to\Zcal$ into a nonempty learned state space. Let $(P_m)_{m\ge1}$ be a query sequence separating states, with $P_m\subseteq\Ucal\setminus Q$. If
  \begin{equation}
    \Lcal_{Q,P_m}^\star(e)\xrightarrow{m\to\infty}0,
  \end{equation}
  then there exists a \textit{state decoder} $a\in\Acal$ such that
  \begin{equation}
    (a\circ e)(f(Q))=f,\qquad\forall f\in\Fcal.
  \end{equation}
\end{corollary}

\begin{proof}
  By~\cref{thm:paper-general-equivalence}, the recovery error tends to zero. Since $Q$ and $e$ are fixed, $\Rcal_Q^\star(e)$ is independent of the stage and therefore equals zero. Apply~\cref{lem:exact-decoding} to obtain an exact \textit{decoder}.
\end{proof}

Embedding theory relates state distinguishability to the geometry preserved by observations~\citep{Sauer1991Embedology}. Here, exact state recovery has a topological counterpart when the state space is compact and the map to learned states is continuous.

\begin{corollary}[Topology Preservation]\label{cor:manifold-recovery}
  Under~\cref{ass:compactness}, fix a query set $Q$ and an \textit{encoder} $e:\Rbb^Q\to\Zcal=\Rbb^\ell$ with a continuous map to learned states $f\mapsto e(f(Q))$. If there exists a \textit{state decoder} $a\in\Acal$ such that
  \begin{equation}
    (a\circ e)(f(Q))=f,\qquad\forall f\in\Fcal,
  \end{equation}
  then this map is a homeomorphism onto its image. In particular, if $\Fcal$ is a topological manifold, possibly with boundary, its image under this map is a manifold of the same dimension.
\end{corollary}

\begin{proof}
  \emph{Injectivity.} By~\cref{lem:exact-decoding}, the map $f\mapsto e(f(Q))$ is injective. It is therefore a continuous bijection onto its image.

  \emph{Continuity of the inverse.} Every closed subset of the compact space $\Fcal$ is compact. Its image is compact by continuity, and therefore closed in $\Rbb^\ell$. The bijection is thus a closed map, so its inverse is continuous. This proves that it is a homeomorphism onto its image.

  \emph{Manifold structure.} Composing each local chart of $\Fcal$ with this inverse gives a local chart on the image, with the same dimension and boundary structure.
  \begin{center}
    \makebox[155.68117pt][c]{\raisebox{-20.00000pt}[68.35419pt][0.00000pt]{\includegraphics[width=195.68117pt,height=108.35419pt]{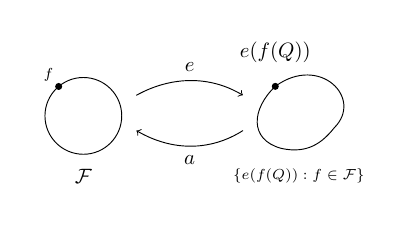}}}

  \end{center}
\end{proof}

\subsection{Low-Dimensional Case}

The general case controls recovery through target ambiguity. In the low-dimensional case, covering and observability let us bound that ambiguity using finitely many queries, yielding the recovery and convergence results in~\cref{thm:paper-low-bounds,thm:paper-low-equivalence}.

\begin{lemma}[Covering Queries]\label{lem:covering-queries}
  Let $\Fcal$ be nonempty and let $(\Fcal,Q)$ satisfy~\cref{ass:low-dimensional-manifold} with constants $C,d,\alpha$. Fix a covering radius $0<\varepsilon\le1$ and a query allowance $\tau>0$.

  There is a finite representative set $\Fcal_\varepsilon\subseteq\Fcal$, with $N=|\Fcal_\varepsilon|\le C\varepsilon^{-d}$, such that every state in $\Fcal$ is within $d_{\Ucal}$-distance $\varepsilon$ of a representative.

  There also exists a target set $P\subseteq\Ucal\setminus Q$, with $1\le|P|\le\max\{1,N(N-1)/2\}$, that distinguishes these representatives: for every $f,g\in\Fcal_\varepsilon$,
  \begin{equation}\label{eq:cover-center-observability}
    d_P\bigl(f(P),g(P)\bigr)\ge\alpha d_{\Ucal}(f,g)-\tau.
  \end{equation}
  Equivalently,
  \begin{equation}\label{eq:cover-center-distance}
    d_{\Ucal}(f,g)\le\frac{1}{\alpha}d_P\bigl(f(P),g(P)\bigr)+\frac{\tau}{\alpha}.
  \end{equation}
  The representatives and targets are selected independently of the \textit{encoder}.
\end{lemma}

\begin{proof}
  \emph{Choose the representatives.} By~\cref{eq:state-covering-growth}, take an $\varepsilon$-cover with centers $\Fcal_\varepsilon\subseteq\Fcal$ and $1\le N=|\Fcal_\varepsilon|\le C\varepsilon^{-d}$.

  \emph{Distinguish each pair.} For each unordered pair $\{f,g\}\subseteq\Fcal_\varepsilon$ of distinct representatives, observability~\cref{eq:target-observability} and $\tau>0$ let us select $p_{\{f,g\}}\in\Ucal\setminus Q$ such that
  \begin{equation}
    |f(p_{\{f,g\}})-g(p_{\{f,g\}})|\ge\alpha d_{\Ucal}(f,g)-\tau.
  \end{equation}
  The positive allowance permits this choice even if no query attains the supremum.

  \emph{Collect the queries.} If $N\ge2$, let $P=\{p_{\{f,g\}}:f,g\in\Fcal_\varepsilon,\ f\ne g\}$. Then
  \begin{equation}
    1\le|P|\le\frac{N(N-1)}{2}.
  \end{equation}
  For distinct $f,g\in\Fcal_\varepsilon$, their selected query belongs to $P$, so
  \begin{equation}
    d_P\bigl(f(P),g(P)\bigr)\ge |f(p_{\{f,g\}})-g(p_{\{f,g\}})| \ge\alpha d_{\Ucal}(f,g)-\tau.
  \end{equation}
  For $f=g$, both distances are zero. This proves~\cref{eq:cover-center-observability}; rearranging gives~\cref{eq:cover-center-distance}.

  If $N=1$, take any singleton $P\subseteq\Ucal\setminus Q$, which exists by~\cref{ass:low-dimensional-manifold}; the distance bounds are immediate.
\end{proof}

We next extend the representative bound to arbitrary states.

\begin{lemma}[Decoding a Cover]\label{lem:cover-decoding}
  Let $\Fcal_\varepsilon\subseteq\Fcal$ be a finite $\varepsilon$-cover, and let a nonempty finite target set $P$ satisfy~\cref{eq:cover-center-distance} with constants $\alpha>0$ and $\tau\ge0$. For $y\in\Rbb^P$, choose a representative $g\in\argmin_{h\in\Fcal_\varepsilon}d_P(y,h(P))$. Then, for every $f\in\Fcal$,
  \begin{equation}\label{eq:cover-decoder-pointwise}
    d_{\Ucal}(f,g)\le\frac2\alpha d_P(y,f(P))+\frac{(2+\alpha)\varepsilon+\tau}{\alpha}.
  \end{equation}
\end{lemma}

\begin{proof}
  Take $f_\varepsilon\in\Fcal_\varepsilon$ with $d_{\Ucal}(f,f_\varepsilon)\le\varepsilon$. By~\cref{eq:cover-center-distance}, the minimizing property of $g$, and restriction of the full metric,
  \begin{equation}
    \begin{aligned}
      d_{\Ucal}(f,g)&\le d_{\Ucal}(f,f_\varepsilon)+d_{\Ucal}(f_\varepsilon,g) \le\varepsilon+\frac1\alpha d_P(f_\varepsilon(P),g(P))+\frac\tau\alpha\\
      &\le\varepsilon+\frac2\alpha d_P(y,f_\varepsilon(P))+\frac\tau\alpha \le\frac2\alpha d_P(y,f(P)) +\frac{(2+\alpha)\varepsilon+\tau}{\alpha}.
    \end{aligned}
  \end{equation}
\end{proof}

To prove~\cref{thm:paper-low-bounds}, we fit this construction to the prescribed query budget and decode to a representative whose target observations are closest to the prediction.

\begin{proof}[Proof of~\cref{thm:paper-low-bounds}]
  \emph{Decode to a representative.} For radius $\varepsilon$ and allowance $\tau$, specified below, \cref{lem:covering-queries} supplies a finite representative set $\Fcal_\varepsilon$ and a target query set $P$. We now use these sets to define a \textit{decoder}. Fix any \textit{encoder} $e$ and finite-error \textit{predictor} $b$. Given a learned state $z$, compare the predicted target vector $b(z)$ with each representative's target vector $g(P)$, and return a closest representative:
  \begin{equation}
    a(z)\in\argmin_{g\in\Fcal_\varepsilon}d_P\bigl(b(z),g(P)\bigr),
  \end{equation}
  resolving ties by a fixed ordering.

  The diagram shows the returned representative $g=a(z)$, whose target readings are closest to $b(z)$.
  \begin{center}
    \noindent\makebox[\linewidth][c]{
      \makebox[192.39635pt][c]{\raisebox{-20.00000pt}[108.81720pt][0.00000pt]{\includegraphics[width=232.39635pt,height=148.81720pt]{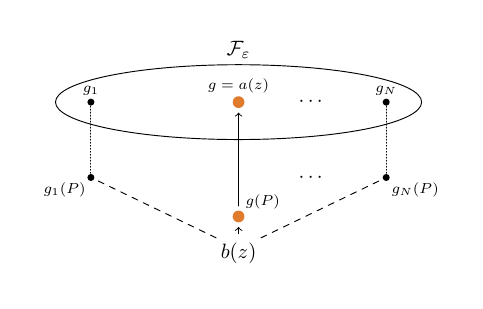}}}

    }
  \end{center}

  \emph{Bound the recovery error.} Apply~\cref{lem:cover-decoding} with $y=(b\circ e)(f(Q))$ and $g=(a\circ e)(f(Q))$:
  \begin{equation}
    d_{\Ucal}(f,(a\circ e)(f(Q)))\le\frac2\alpha d_P((b\circ e)(f(Q)),f(P)) +\frac{(2+\alpha)\varepsilon+\tau}{\alpha}.
  \end{equation}

  \emph{Set the query budget.} The radius controls the number of representatives; the allowance absorbs the approximation of separating queries. Let $\delta_m=(C^2/(2m))^{1/(2d)}\le1$. Choose $0<\varepsilon<\delta_m$ sufficiently close to $\delta_m$ that
  \begin{equation}
    N\le C\varepsilon^{-d}<\lfloor C\delta_m^{-d}\rfloor+1 =\lfloor\sqrt{2m}\rfloor+1,
  \end{equation}
  hence $N\le\lfloor\sqrt{2m}\rfloor$. Set $\tau=(2+\alpha)(\delta_m-\varepsilon)>0$. By~\cref{lem:covering-queries},
  \begin{equation}
    |P|\le\max\left\{1,\frac{N(N-1)}{2}\right\} \le\max\left\{1,\frac{N^2}{2}\right\}\le m.
  \end{equation}
  These choices are independent of $e$. Taking the supremum of the recovery-error bound gives
  \begin{equation}
    \begin{aligned}
      \Rcal_Q^\star(e)&\le\Rcal_Q(a;e) \le\frac{2}{\alpha}\Lcal_{Q,P}(b;e)+\frac{(2+\alpha)\varepsilon+\tau}{\alpha}\\
      &=\frac{2}{\alpha}\Lcal_{Q,P}(b;e) +\frac{(2+\alpha)\varepsilon+(2+\alpha)(\delta_m-\varepsilon)}{\alpha} =\frac{2}{\alpha}\Lcal_{Q,P}(b;e)+\frac{2+\alpha}{\alpha}\delta_m.
    \end{aligned}
  \end{equation}
  Taking the infimum over $b$ proves the upper bound. The lower bound follows from~\cref{thm:paper-general-bounds}, since the covering assumption makes $\Fcal$ bounded.
\end{proof}

The bound in~\cref{thm:paper-low-bounds} now gives the convergence and rate equivalences. These target sets need not be nested: each budget has its own uniform recovery bound.

\begin{proof}[Proof of~\cref{thm:paper-low-equivalence}]
  By~\cref{thm:paper-low-bounds}, for every admissible budget,
  \begin{equation}\label{eq:low-equivalence-sandwich}
    0\le\Lcal_{Q_m,P_m}^\star(e_m)\le\Rcal_{Q_m}^\star(e_m) \le\frac{2}{\alpha}\Lcal_{Q_m,P_m}^\star(e_m) +\frac{2+\alpha}{\alpha}\left(\frac{C^2}{2m}\right)^{1/(2d)}.
  \end{equation}
  \emph{Convergence.}

  \emph{Left-to-right.} If $\Lcal_{Q_m,P_m}^\star(e_m)\to0$, both terms in the upper bound tend to zero, so $\Rcal_{Q_m}^\star(e_m)\to0$.

  \emph{Right-to-left.} If $\Rcal_{Q_m}^\star(e_m)\to0$, the lower bound gives $\Lcal_{Q_m,P_m}^\star(e_m)\to0$.

  \emph{Rate.}

  \emph{Left-to-right.} Assume $\Lcal_{Q_m,P_m}^\star(e_m)=O(m^{-1/(2d)})$. By definition, there is a constant $c\ge0$, independent of $m$, such that $\Lcal_{Q_m,P_m}^\star(e_m)\le c\,m^{-1/(2d)}$ for all sufficiently large $m$. Substituting this into~\cref{eq:low-equivalence-sandwich} gives
  \begin{equation}
    \begin{aligned}
      \Rcal_{Q_m}^\star(e_m)&\le\frac2\alpha\Lcal_{Q_m,P_m}^\star(e_m) +\frac{2+\alpha}{\alpha}\left(\frac{C^2}{2m}\right)^{1/(2d)}\\
      &\le\frac{2c}{\alpha}m^{-1/(2d)} +\frac{2+\alpha}{\alpha}\left(\frac{C^2}{2}\right)^{1/(2d)}m^{-1/(2d)}\\
      &=\left[\frac{2c}{\alpha}+\frac{2+\alpha}{\alpha}\left(\frac{C^2}{2}\right)^{1/(2d)}\right] \cdot m^{-1/(2d)}.
    \end{aligned}
  \end{equation}
  The bracket is independent of $m$, so $\Rcal_{Q_m}^\star(e_m)=O(m^{-1/(2d)})$.

  \emph{Right-to-left.} Assume $\Rcal_{Q_m}^\star(e_m)=O(m^{-1/(2d)})$. There is a constant $c\ge0$, independent of $m$, such that, for all sufficiently large $m$,
  \begin{equation}
    0\le\Lcal_{Q_m,P_m}^\star(e_m)\le\Rcal_{Q_m}^\star(e_m)\le c\,m^{-1/(2d)}.
  \end{equation}
  Thus $\Lcal_{Q_m,P_m}^\star(e_m)=O(m^{-1/(2d)})$, completing~\cref{eq:paper-low-rate-equivalence}.
\end{proof}

For a fixed \textit{encoder}, the selected targets give the same recovery conclusion.

\begin{corollary}[Low-Dimensional Exact Recovery]\label{cor:low-dimensional-state-recovery}
  Under the assumptions of~\cref{thm:paper-low-bounds}, fix an \textit{encoder} $e:\Rbb^Q\to\Zcal$ into a nonempty learned state space. For each integer $m\ge C^2/2$, choose $P_m$ as in that theorem for the same input set $Q$. If
  \begin{equation}
    \Lcal_{Q,P_m}^\star(e)\xrightarrow{m\to\infty}0,
  \end{equation}
  then there exists a \textit{state decoder} $a\in\Acal$ such that
  \begin{equation}
    (a\circ e)(f(Q))=f,\qquad\forall f\in\Fcal.
  \end{equation}
\end{corollary}

\begin{proof}
  By~\cref{thm:paper-low-equivalence}, the recovery error tends to zero. Since $Q$ and $e$ are fixed, $\Rcal_Q^\star(e)$ is independent of the stage and therefore equals zero. Apply~\cref{lem:exact-decoding} to obtain an exact \textit{decoder}.
\end{proof}

\subsection{Budget Sharpness}\label{sec:sharpness-topology}\label{sec:query-budget-sharpness}

We first prove the fixed-family result supporting~\cref{thm:paper-fixed-transfer}, then give a finite-query construction for comparison.

\mypara{A fixed family at every budget.}
A multiscale arc places the obstruction inside a fixed family: finer blocks contain more distinctions, while their response amplitudes decrease. The \textit{encoder} merges a pair whose distinguishing query was omitted, and the radius of that pair gives the recovery cost.

We construct the fixed family in~\cref{lem:multiscale-arc}, find a pair missed by the target set in~\cref{lem:omitted-pair}, and merge that pair in~\cref{lem:two-state-fiber}. Combining their bounds proves~\cref{lem:fixed-family-budget}.

\begin{lemma}[Multiscale Arc]\label{lem:multiscale-arc}
  For every $d>1$, there exist a countable query universe $\Ucal$, a singleton input set $Q$, and a compact state family $\Fcal$ homeomorphic to $[0,1]$, of diameter $1$. The input observation map is injective, and $(\Fcal,Q)$ satisfies~\cref{ass:low-dimensional-manifold} with $C=C_d$ and $\alpha=1$.
  The covering constant is
  \begin{equation}
    C_d\coloneqq1+\frac{12}{1-2^{1/d-1}}.
  \end{equation}
\end{lemma}

\begin{proof}
  \raggedbottom
  \emph{Construct the family.} Fix $d>1$. For $j\ge1$, put
  \begin{align}
    N_j&=8\cdot2^{j-1}, & s_j&=\frac12\,2^{-(j-1)/d},\\
    \ell_j&=2^{-j-1}, & T_j&=1-2^{1-j}.
  \end{align}
  Let $v_1,\ldots,v_{N_j}$ denote the standard basis in $\Rbb^{N_j}$. On $[T_j,T_j+\ell_j]$, the vector $\lambda^j(t)$ follows the polygonal path $v_1\to\cdots\to v_{N_j}$, spending equal time on each edge; the vectors in other blocks are zero. On $[T_j+\ell_j,T_{j+1}]$, join this block directly to the next: set $\lambda^j(t)=w v_{N_j}$ and $\lambda^{j+1}(t)=(1-w)v_1$, where $w$ decreases linearly from $1$ to $0$. At $t=1$, all block vectors are zero. These intervals cover $[0,1)$ because $T_{j+1}-T_j=2\ell_j$.

  Take $Q=\{q_0\}$. The remaining queries are a duplicate parameter query $p_0$, a height query $p^j_{\mathrm h}$ for each block, and a private query $p^j_{ab}$ for each $1\le a<b\le N_j$. Define the state $f_t$, for $0\le t\le1$, by
  \begin{align}
    f_t(q_0)=f_t(p_0)&=t/2,\\
    f_t(p^j_{ab})&=s_j\bigl(\lambda_a^j(t)-\lambda_b^j(t)\bigr),\\
    f_t(p^j_{\mathrm h})&=2s_j\bigl(1-\|\lambda^j(t)\|_1\bigr).
  \end{align}
  Let $\Fcal=\{f_t:0\le t\le1\}$.

  \emph{Geometry and observability.} All responses are bounded. The parameter query makes $t\mapsto f_t$ injective, and its duplicate gives target observability with $\alpha=1$. Each finite part of the path is continuous in the full response metric. From block $j$ onward, the distance to $f_1$ is at most $2s_j$: active private coordinates have magnitude at most $s_j$, height differences are at most $2s_j$, and the parameter difference is at most $2^{-j}\le s_j$. Thus the path is also continuous at $1$. Compactness and injectivity make it a homeomorphism onto $\Fcal$. Every coordinate has range at most $1$, and two vertices in the first block differ by $2s_1=1$, so the diameter is exactly $1$.

  \noindent\begin{minipage}{\linewidth}
    The shrinking blocks accumulate at the endpoint.
    \begin{center}
      \makebox[147.37401pt][c]{\raisebox{-20.00000pt}[38.81730pt][0.00000pt]{\includegraphics[width=187.37401pt,height=78.81730pt]{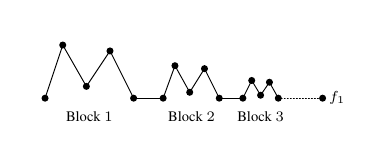}}}

    \end{center}
  \end{minipage}

  \emph{Covering bound.} Each edge within block $j$ is a straight response-space segment of length $2s_j$, as is its outgoing connector. Indeed, a private coordinate attains this length on a block edge, and the height query $p^j_{\mathrm h}$ attains it on the connector; the parameter increment is smaller. The $N_j-1$ block edges and one outgoing connector give $N_j$ closed segments per block, including their endpoints. For $0<\varepsilon<1$, let $J$ be the last index with $2s_J>\varepsilon$. The tail beginning at block $J+1$ lies in the closed $\varepsilon$-ball about $f_1$. A segment of length $2s_j$ is covered by $\lceil s_j/\varepsilon\rceil$ closed balls centered on the segment. Consequently,
  \begin{equation}
    \begin{aligned}
      \Ncal_{\Fcal}(\varepsilon)&\le1+\sum_{j\le J}N_j\left(\frac{s_j}{\varepsilon}+1\right) \le1+\frac3\varepsilon\sum_{j\le J}N_js_j,\\
      \sum_{j\le J}N_js_j&\le\frac{N_Js_J}{1-2^{1/d-1}} \le\frac{4\varepsilon^{1-d}}{1-2^{1/d-1}}.
    \end{aligned}
  \end{equation}
  This gives the stated $C_d\varepsilon^{-d}$ bound. At $\varepsilon=1$, one ball suffices. Moreover, the $N_j$ vertices of block $j$ are mutually $2s_j$ apart. At radius $s_j/2$, each covering ball contains at most one of them. Since $N_j=8(2s_j)^{-d}$, the covering exponent is exactly $d$.
\end{proof}

The next two lemmas use the family and block notation constructed in the proof of~\cref{lem:multiscale-arc}.

\begin{lemma}[Omitted Pair]\label{lem:omitted-pair}
  For every integer $m\ge1$ and nonempty $P\subseteq\Ucal\setminus Q$ with $|P|\le m$, some block $j$ contains two nonadjacent vertices with parameters $t_i,t_k$ such that
  \begin{equation}
    d_P(f_{t_i}(P),f_{t_k}(P))\le s_j,\qquad s_j\ge\frac12m^{-1/(2d)}.
  \end{equation}
\end{lemma}

\begin{proof}
  Fix $m$ and $P$. Choose the smallest $j$ with $(N_j-1)(N_j-2)/2>m$. This is the number of nonadjacent vertex pairs in block $j$, so some pair $i<k$, with $k-i\ge2$, has $p^j_{ik}\notin P$. Write $t_i,t_k$ for their parameter values. Their private query $p^j_{ik}$ has response gap $2s_j$. Every other private query has gap at most $s_j$, every height query has gap zero, and the parameter gap is at most $\ell_j/2\le s_j/4$. Hence
  \begin{equation}
    d_P(f_{t_i}(P),f_{t_k}(P))\le s_j.
  \end{equation}
  When $j=1$, $N_j=8\le8\sqrt m$. Otherwise minimality of $j$ gives $(N_{j-1}-1)(N_{j-1}-2)/2\le m$. Hence $N_{j-1}\le\sqrt{2m}+2$ and $N_j=2N_{j-1}\le8\sqrt m$. It follows that
  \begin{equation}\label{eq:fixed-family-scale}
    \frac{s_j}{3}=\frac16(N_j/8)^{-1/d}\ge\frac16m^{-1/(2d)}.
  \end{equation}
\end{proof}

\begin{lemma}[Two-State Fiber]\label{lem:two-state-fiber}
  Fix two nonadjacent vertices $f_{t_i},f_{t_k}$ with $i<k$ in block $j$ and a nonempty finite target set $P\subseteq\Ucal\setminus Q$. There is a globally Lipschitz \textit{encoder} $e:\Rbb^Q\to\Rbb^2$ whose only non-singleton state fiber is this pair, with
  \begin{align}
    \Lcal_{Q,P}^\star(e)&=\frac12d_P(f_{t_i}(P),f_{t_k}(P)),\\
    \Rcal_Q^\star(e)&\ge\frac43s_j.
  \end{align}
\end{lemma}

\begin{proof}
  \emph{A two-state fiber.} Put $u_i=t_i/2$ and $u_k=t_k/2$. Define
  \begin{equation}
    e(u)=\!\!
    \begin{cases}
      (u-u_i,0),&u\le u_i,\\
      (\sin\!\frac{2\pi(u-u_i)}{u_k-u_i},1-\cos\!\frac{2\pi(u-u_i)}{u_k-u_i}),&u_i\le u\le u_k,\\
      (u-u_k,0),&u\ge u_k.
    \end{cases}
  \end{equation}
  \begin{center}
    \makebox[128.48730pt][c]{\raisebox{-20.00000pt}[57.21136pt][0.00000pt]{\includegraphics[width=168.48730pt,height=97.21136pt]{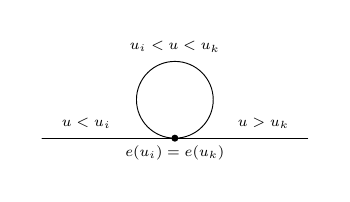}}}

  \end{center}
  The circular part has positive second coordinate in its interior, while the two tails lie on opposite halves of the horizontal axis. Thus the only merged input pair is $u_i,u_k$. The pieces agree at their endpoints, and the \textit{encoder} is globally Lipschitz with constant at most $\max\{1,2\pi/(u_k-u_i)\}$. By~\cref{lem:optimal-prediction-error},
  \begin{equation}
    \Lcal_{Q,P}^\star(e)=\frac12d_P(f_{t_i}(P),f_{t_k}(P)).
  \end{equation}

  \emph{The admissible recovery center.} Every singleton fiber can be decoded exactly. It remains to bound the radius of the pair about any $g\in\Fcal$.

  For $g=f_t$ in block $j$, the vector $\lambda^j(t)$ has mass one on at most two adjacent vertices. Exchange $i,k$ if necessary so that $\lambda_i^j(t)=0$, and take a largest coordinate $\lambda_b^j(t)\ge1/2$. Then $b\ne i$, and the private query indexed by $\{i,b\}$ gives response difference
  \begin{equation}
    s_j\bigl|1-\lambda_i^j(t)+\lambda_b^j(t)\bigr|=s_j\bigl(1+\lambda_b^j(t)\bigr)\ge\frac32s_j.
  \end{equation}
  On an adjacent connector, the block vector is $w$ times an endpoint, giving private-coordinate error at least $s_j(1+w)$ and height error $2s_j(1-w)$. Elsewhere, the block mass is zero and the height error is $2s_j$. Thus
  \begin{equation}
    \frac1{s_j}\max\{d_{\Ucal}(g,f_{t_i}),d_{\Ucal}(g,f_{t_k})\}\ge\begin{cases}
      3/2,&\text{in block }j,\\
      \max\{1+w,2(1-w)\},&\text{on a connector},\\
      2,&\text{elsewhere}
    \end{cases}\ge\frac43.
  \end{equation}
  The connector bound follows by minimizing over $0\le w\le1$, with equality at $w=1/3$. Taking the infimum over common decoded states gives the recovery bound.
\end{proof}

The three lemmas now give the fixed-family lower bound.

\begin{lemma}[Fixed-Family Sharpness]\label{lem:fixed-family-budget}
  For every $d>1$, there exist a fixed countable query universe $\Ucal$, a singleton input set $Q$, and a compact state family $\Fcal$ homeomorphic to $[0,1]$, of diameter $1$. The input observation map is injective, and $(\Fcal,Q)$ satisfies~\cref{ass:low-dimensional-manifold} with $C=C_d$ and $\alpha=1$.

  For every integer $m\ge1$ and every nonempty $P\subseteq\Ucal\setminus Q$ with $|P|\le m$, there is a globally Lipschitz \textit{encoder} $e:\Rbb^Q\to\Rbb^2$ such that
  \begin{equation}\label{eq:fixed-family-budget}
    \Rcal_Q^\star(e)-2\Lcal_{Q,P}^\star(e)\ge\frac16m^{-1/(2d)}.
  \end{equation}
\end{lemma}

\begin{proof}
  \Cref{lem:multiscale-arc} supplies the fixed family and its geometric properties. For any admissible $m,P$, take the pair from~\cref{lem:omitted-pair} and the \textit{encoder} from~\cref{lem:two-state-fiber}. Then
  \begin{equation}
    \Rcal_Q^\star(e)-2\Lcal_{Q,P}^\star(e)\ge\frac43s_j-d_P(f_{t_i}(P),f_{t_k}(P)) \ge\frac{s_j}{3}\ge\frac16m^{-1/(2d)}.
  \end{equation}
\end{proof}

\mypara{The uniform transfer residual.}
For the fixed family above, the uniform transfer residual in~\cref{def:uniform-transfer-residual} measures the recovery cost of choosing targets before the \textit{encoder}. The lower bound is supplied by~\cref{lem:fixed-family-budget}; the upper bound follows from~\cref{thm:paper-low-bounds}.

\begin{proof}[Proof of~\cref{thm:paper-fixed-transfer}]
  Use the fixed family in~\cref{lem:fixed-family-budget}, and fix an integer $m\ge C_d^2/2$.

  \emph{Lower bound.} For every admissible target set $P$, \cref{lem:fixed-family-budget} supplies a globally Lipschitz \textit{encoder} witnessing
  \begin{equation}
    \sup_{e\in\mathcal E_{\mathrm{Lip}}}\bigl[\Rcal_Q^\star(e)-2\Lcal_{Q,P}^\star(e)\bigr]\ge\frac16m^{-1/(2d)}.
  \end{equation}
  Taking the infimum over $P$ gives the lower bound for $\mathcal T_m(\Fcal,Q)$.

  \emph{Upper bound.} Apply~\cref{thm:paper-low-bounds} with $C=C_d$ and $\alpha=1$. It supplies one admissible target set $P$ for which every \textit{encoder} satisfies
  \begin{equation}
    \Rcal_Q^\star(e)-2\Lcal_{Q,P}^\star(e)\le3\left(\frac{C_d^2}{2m}\right)^{1/(2d)}.
  \end{equation}
  Taking the supremum over $\mathcal E_{\mathrm{Lip}}$ for this $P$, then the infimum over all admissible target sets, gives the upper bound.

  \emph{$1$-Lipschitz encoders.} If $e$ has a Lipschitz bound $K\ge1$, then $e/K$ is $1$-Lipschitz. Composing a readout with $z\mapsto Kz$ transfers it from $e$ to $e/K$; $z\mapsto z/K$ reverses this correspondence. Predictions and recovered states are preserved, so both optimal errors are unchanged. Every residual value is thus attained in the smaller class; inclusion gives the reverse inequality between suprema.
\end{proof}

\mypara{Recovery and sampling.}
Classical optimal recovery describes the best readout from fixed information by the radius of its consistency sets~\citep[Chapters~1 and~3]{Traub1983Information}. Here the information is $e(f(Q))$, and the same consistency set determines two errors: its radius in the admissible state family and its radius in the target observation space. The fixed arc makes their difference persist at every finite target budget.

Optimal sampling instead chooses queries together with a reconstruction map~\citep[Sections~2 and~9.1]{Krieg2026Sampling}. Write the arc as $\Fcal=\{f_t:0\le t\le1\}$, with parameter query $p_0\in\Ucal\setminus Q$ satisfying $f_t(p_0)=t/2$. This query determines the state exactly. With $a(y)=f_{2y}$ for $0\le y\le1/2$, extended by $f_0$ elsewhere,
\begin{equation}\label{eq:fixed-family-exact-sampling}
  \inf_{a:\Rbb\to\Fcal}\sup_{0\le t\le1}d_{\Ucal}\bigl(a(f_t(p_0)),f_t\bigr)=0.
\end{equation}
The residual in~\cref{eq:paper-transfer-residual} measures the budget needed to transfer prediction accuracy to recovery uniformly across \textit{encoders}. On the constructed family, a residual tolerance $\varepsilon$ requires a budget of order $\varepsilon^{-2d}$, while direct reconstruction from the parameter query is exact.

Uniform-norm discretization gives a related view of how finitely many queries preserve distinctions in a linear function space~\citep[Theorems~1.1 and~1.2]{Kosov2024UniformDiscretization}. Its criterion is multiplicative norm preservation over that space. The criterion here couples prediction and state-valued recovery on a fixed nonlinear compact family. The covering bound controls the number of states needed at a given resolution; the pairwise witnesses in the construction determine the target budget.

\mypara{Encoder regularity and target order.}
\Cref{thm:paper-fixed-transfer} gives a common \textit{encoder} regularity bound. Keeping this bound fixed, we compare targets selected before and after the \textit{encoder}.

\begin{corollary}[Target Selection Order]\label{cor:regularity-target-order}
  On the fixed state family of~\cref{lem:fixed-family-budget}, let $\mathcal E_1$ consist of the $1$-Lipschitz maps $e:\Rbb^Q\to\Rbb^2$. For every $e\in\mathcal E_1$, there is a query $p\in\Ucal\setminus Q$ such that $\Rcal_Q^\star(e)\le2\Lcal_{Q,\{p\}}^\star(e)$. Consequently, for every integer $m\ge1$,
  \begin{equation}\label{eq:encoder-dependent-transfer}
    \sup_{e\in\mathcal E_1}\inf_{\substack{P\subseteq\Ucal\setminus Q\\1\le|P|\le m}}\bigl[\Rcal_Q^\star(e)-2\Lcal_{Q,P}^\star(e)\bigr]=0.
  \end{equation}
\end{corollary}

\begin{proof}
  For a fixed $e\in\mathcal E_1$, put
  \begin{equation}
    \Delta_e\coloneqq\max_{\substack{f,g\in\Fcal\\e(f(Q))=e(g(Q))}}d_{\Ucal}(f,g).
  \end{equation}
  The maximum exists: the map $f\mapsto e(f(Q))$ is continuous, so its equality set is a nonempty closed subset of $\Fcal\times\Fcal$. By~\cref{lem:fiberwise-readout}, a state-valued representative \textit{decoder} gives $\Rcal_Q^\star(e)\le\Delta_e$.

  Each state on the arc has nonzero block mass in at most two blocks. For any pair, differences in private and height responses are confined to the finite union of their active blocks. The remaining differing coordinates are $q_0,p_0$, whose differences coincide. A query outside $Q$ therefore attains the full distance of the pair. Apply this to a pair attaining $\Delta_e$, and call its query $p$. The largest $p$-response gap among pairs sharing \textit{encoder} output equals $\Delta_e$: this pair attains it, and every such gap is bounded by the full-state gap. By~\cref{lem:optimal-prediction-error},
  \begin{equation}
    \Rcal_Q^\star(e)\le\Delta_e=2\Lcal_{Q,\{p\}}^\star(e).
  \end{equation}
  When $\Delta_e=0$, the same conclusion holds with any available query. Each \textit{encoder} thus has a singleton target giving a nonpositive residual. Conversely, $e(u)=(u,0)$ belongs to $\mathcal E_1$, and the input observation identifies every state. For this \textit{encoder} both optimal errors are zero for every target set. Taking the indicated infimum and supremum proves~\cref{eq:encoder-dependent-transfer}.
\end{proof}

\mypara{A finite-query comparison.}
Allowing the family to depend on the budget yields a complementary construction: the \textit{encoder} is linear, the query universe is finite, and the exponent range includes $d=1$. The same order is attained at prediction-error coefficient $2$ and $\alpha=1$.

\begin{proposition}[Budget Sharpness]\label{prop:query-budget-sharpness}
  For every $d\ge1$, there is a constant $c_d>0$ such that, for each integer $m\ge1$, there exist a finite query universe $\Ucal$, a query set $Q$, and a compact connected one-dimensional state manifold $\Fcal$ of diameter $1$, such that $(\Fcal,Q)$ satisfies~\cref{ass:low-dimensional-manifold} with $C=3$ and $\alpha=1$. The input observation map is injective. For every nonempty $P\subseteq\Ucal\setminus Q$ with $|P|\le m$, there is a linear \textit{encoder} into a Euclidean learned state space such that
  \begin{equation}
    \Rcal_Q^\star(e)\ge2\Lcal_{Q,P}^\star(e)+c_d m^{-1/(2d)}.
  \end{equation}
\end{proposition}

\begin{proof}
  \emph{State space.} Fix $d\ge1$ and $m\ge1$, and set
  \begin{equation}
    N\coloneqq\lceil\sqrt{2m}\rceil+3, \qquad s\coloneqq\frac12N^{-1/d}.
  \end{equation}
  Let $v_1,\ldots,v_N$ be the standard basis of $\Rbb^N$, with $v_{N+1}=v_1$. On the polygonal cycle $\bigcup_{i=1}^N[v_i,v_{i+1}]$, let $H(\lambda)=2\min\{\lambda_1,\lambda_2\}$. This continuous function is zero off the edge $[v_1,v_2]$ and reaches $1$ at its midpoint. Take $Q=\{q_0,q_1,\ldots,q_N\}$ and $\Ucal\setminus Q=\{p_0\}\cup\{p_{ij}:1\le i<j\le N\}$. For each $\lambda$ on the cycle, define, for $1\le i\le N$ and $1\le i<j\le N$,
  \begin{equation}
    f_\lambda(q_i)\coloneqq s\lambda_i, \qquad f_\lambda(p_{ij})\coloneqq s(\lambda_i-\lambda_j).
  \end{equation}
  Also set $f_\lambda(q_0)=f_\lambda(p_0)=H(\lambda)$, lifting one edge while leaving all vertices unchanged. Let $\Fcal$ consist of these states, and write $f_i=f_{v_i}$. The map $\lambda\mapsto f_\lambda$ is continuous, with continuous inverse $\lambda_i=f_\lambda(q_i)/s$. Thus $\Fcal$ is homeomorphic to the cycle, a compact connected one-dimensional manifold, and the input observation map is injective. Since $\sum_i\lambda_i=\sum_i\mu_i=1$, we have
  \begin{equation}\label{eq:cycle-state-distance}
    d_{\Ucal}(f_\lambda,f_\mu)=\max\bigl\{s[\max_k(\lambda_k-\mu_k)-\min_k(\lambda_k-\mu_k)], |H(\lambda)-H(\mu)|\bigr\}.
  \end{equation}
  The target queries capture both terms, so target observability holds with $\alpha=1$. Since $2s\le1$ and $H$ ranges from $0$ to $1$, the diameter is exactly $1$.

  \emph{Covering bound.} Each unlifted edge has length $2s$ in $d_{\Ucal}$. For $0<\varepsilon<s$, the unlifted part can be covered by at most $N\lceil s/\varepsilon\rceil$ balls centered along these edges. Since $d\ge1$,
  \begin{equation}
    N\lceil s/\varepsilon\rceil\le\frac{2Ns}{\varepsilon} \le2Ns^d\varepsilon^{-d}=2^{1-d}\varepsilon^{-d}\le\varepsilon^{-d}.
  \end{equation}
  For $s\le\varepsilon<2s$, the $N$ vertices cover the unlifted part and $N=(2s)^{-d}\le\varepsilon^{-d}$. For $2s\le\varepsilon\le1$, one ball centered at a vertex covers that part, and $1\le\varepsilon^{-d}$. The lifted edge consists of two segments of length $1$ each, so it admits a $1$-Lipschitz parametrization on an interval of length $2$. It can therefore be covered by $\lceil1/\varepsilon\rceil\le2\varepsilon^{-1}$ additional balls centered along it. Consequently,
  \begin{equation}
    \Ncal_{\Fcal}(\varepsilon)\le\varepsilon^{-d}+2\varepsilon^{-1}\le3\varepsilon^{-d}, \qquad 0<\varepsilon\le1.
  \end{equation}

  \emph{Encoder construction.} Fix any allowed $P$. The cycle has $N(N-3)/2>m$ unordered nonadjacent vertex pairs, so some such pair $i<j$ has $p_{ij}\notin P$. Define a linear \textit{encoder} $e:\Rbb^Q\to\Zcal=\Rbb^N$ by
  \begin{equation}
    e(x)\coloneqq\bigl(x(q_0),\,x(q_i)+x(q_j),\,(x(q_k))_{k\in\{1,\ldots,N\}\setminus\{i,j\}}\bigr).
  \end{equation}
  It combines the two corresponding input coordinates and retains all others. In particular, $f_i$ and $f_j$ have the same learned state. The diagram illustrates the nonadjacent pair and its common learned state.
  \begin{center}
    \noindent\makebox[\linewidth][c]{\makebox[169.05710pt][c]{\raisebox{-20.00000pt}[92.03640pt][0.00000pt]{\includegraphics[width=209.05710pt,height=132.03640pt]{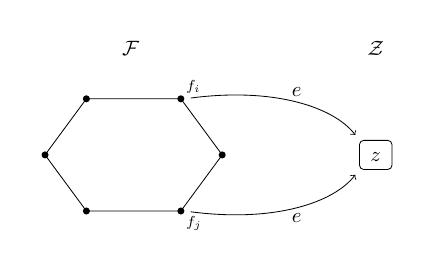}}}
}
  \end{center}

  \emph{Recovery lower bound.} Any \textit{state decoder} must give one output $g=f_\lambda\in\Fcal$ to the common learned state of $f_i,f_j$. A point on the cycle has at most two nonzero coordinates, at adjacent vertices. Since $i,j$ are nonadjacent, at least one of $\lambda_i,\lambda_j$ is zero. Also $\max_k\lambda_k\ge1/2$. By~\cref{eq:cycle-state-distance},
  \begin{equation}
    \max\{d_{\Ucal}(f_i,g),d_{\Ucal}(f_j,g)\}\ge s\bigl(1+\max_k\lambda_k\bigr) \ge\frac{3s}{2}.
  \end{equation}
  Thus every \textit{decoder} has recovery error at least $3s/2$, so $\Rcal_Q^\star(e)\ge3s/2$.

  \emph{Prediction upper bound.} If $e(f_\lambda(Q))=e(f_\mu(Q))$, then $\lambda-\mu=t(v_i-v_j)$ for some $|t|\le1$, and $H(\lambda)=H(\mu)$. The query $p_{ij}$ differs by $2s|t|$; every other pair query differs by at most $s|t|$, and $p_0$ has identical observations. Since $p_{ij}\notin P$, \cref{lem:optimal-prediction-error} gives
  \begin{equation}
    \Lcal_{Q,P}^\star(e)\le\frac{s}{2}.
  \end{equation}
  Combining the two error bounds and using $N\le\sqrt{2m}+4\le6\sqrt m$, we obtain
  \begin{equation}
    \Rcal_Q^\star(e)-2\Lcal_{Q,P}^\star(e)\ge\frac{s}{2}=\frac14N^{-1/d} \ge\frac{1}{4\cdot6^{1/d}}m^{-1/(2d)}.
  \end{equation}
  Taking $c_d=1/(4\cdot6^{1/d})$ completes the proof.
\end{proof}

\section{Fine-Tuning and Recovery}\label{sec:finetuning-state-recovery}

The preceding results start from prediction error over the entire state family. Fine-tuning instead gives an error on training observations. Coverage and prediction stability connect these quantities: a nearby training record controls the prediction at an unobserved state. The general and low-dimensional bounds then transfer that control to recovery.

\subsection{Fine-Tuning Setup}

We work under~\cref{ass:compactness}. Fix disjoint query sets $Q,P$ and maps $e:\Rbb^Q\to\Zcal$, $b:\Zcal\to\Rbb^P$, where $\Zcal$ is nonempty. We first control this pair's prediction error from its errors on the training records.

Use the recorded error and observation coverage radius from~\cref{def:recorded-error-coverage}, together with~\cref{ass:uniform-prediction-stability}. By~\cref{eq:observation-contraction}, $h_n$ is no larger than the coverage radius in the full ontic metric. The asymptotic results use a common stability modulus across the returned outputs.

\begin{lemma}[Record Extension]\label{lem:record-extension}
  Let $(\Fcal,d_{\Ucal})$ be nonempty and compact, and let $f_1,\ldots,f_n\in\Fcal$ with $n\ge1$ index the training records. For disjoint query sets $Q,P$, every pair $(e,b)$ satisfying~\cref{ass:uniform-prediction-stability} obeys
  \begin{equation}\label{eq:empirical-to-uniform-prediction}
    \widehat{\Lcal}_{Q,P,n}(b;e)\le\boxed{\Lcal_{Q,P}(b;e)}\le \widehat{\Lcal}_{Q,P,n}(b;e) +\rho(h_n)+h_n,
  \end{equation}
  where $h_n$ is the observation coverage radius in~\cref{eq:training-coverage}.
\end{lemma}

\begin{proof}
  The empirical error is a maximum over states in $\Fcal$, so it cannot exceed the full-family error. For the upper bound, fix $f\in\Fcal$ and choose a nearest training observation pair. Its generating state $f_i$ satisfies $d_{Q\cup P}(f(Q\cup P),f_i(Q\cup P))\le h_n$. Prediction stability and the triangle inequality give
  \begin{equation}
    \begin{aligned}
      d_P\bigl((b\circ e)(f(Q)),f(P)\bigr)&\le d_P\bigl((b\circ e)(f(Q)),(b\circ e)(f_i(Q))\bigr) +d_P\bigl((b\circ e)(f_i(Q)),f_i(P)\bigr)\\
      &\quad+d_P\bigl(f_i(P),f(P)\bigr) \le\rho(h_n)+\widehat{\Lcal}_{Q,P,n}(b;e)+h_n.
    \end{aligned}
  \end{equation}
  Taking the supremum over $f$ proves the upper bound. Both the empirical maximum and the coverage radius are finite, so the full-family error is finite as well.
\end{proof}

\begin{proof}[Proof of~\cref{prop:record-recovery}]
  \Cref{lem:record-extension} gives the prediction bound. Since the optimal prediction error is no larger, \cref{thm:paper-general-bounds,thm:paper-low-bounds} and monotonicity of $\omega_P$ give the two recovery bounds.
\end{proof}

We now choose the pair by fine-tuning. Let $\Hcal$ be the nonempty class of \textit{encoder--predictor} pairs accessible from the initialization by changing the permitted parameters. The empirical error is defined in~\cref{eq:empirical-record-error}.

\begin{definition}[Fine-Tuning]\label{def:empirical-finetuning}
  The \emph{optimal prediction error in the fine-tuning class} is
  \begin{equation}
    \Lcal_{Q,P}^{\Hcal,\star}\coloneqq\inf_{(e,b)\in\Hcal}\Lcal_{Q,P}(b;e).
  \end{equation}
  Given an optimization tolerance $\eta\ge0$, an \emph{$\eta$-optimal fine-tuning output} is a pair $(\hat e,\hat b)\in\Hcal$ satisfying
  \begin{equation}\label{eq:empirical-finetuning-rule}
    \widehat{\Lcal}_{Q,P,n}(\hat b;\hat e)\le\inf_{(e,b)\in\Hcal}\widehat{\Lcal}_{Q,P,n}(b;e)+\eta.
  \end{equation}
\end{definition}

For every $\eta>0$, an $\eta$-optimal output exists by the definition of the infimum. At $\eta=0$, the rule selects an empirical minimizer whenever the infimum is attained. The quantity $\Lcal_{Q,P}^{\Hcal,\star}$ optimizes jointly over the permitted \textit{encoder--predictor} pairs, while $\Lcal_{Q,P}^\star(e)$ holds the \textit{encoder} fixed.

For a fine-tuning output, the same lemma combines with empirical near-optimality to relate its prediction error to the best error in the model class.

\begin{corollary}[Fine-Tuning Consistency]\label{cor:finetuning-prediction}
  Every $\eta$-optimal fine-tuning output satisfying~\cref{ass:uniform-prediction-stability} obeys
  \begin{equation}\label{eq:finetuning-prediction-bound}
    \Lcal_{Q,P}^{\Hcal,\star}\le\boxed{\Lcal_{Q,P}(\hat b;\hat e)}\le \Lcal_{Q,P}^{\Hcal,\star}+\eta +\rho(h_n)+h_n.
  \end{equation}
  For stages $k\ge1$, allow the training data, query sets, model classes, and learned state spaces to vary, with a common modulus $\rho$ for the returned outputs. Let $(\hat e_k,\hat b_k)$ be $\eta_k$-optimal outputs and let $h_k$ be the corresponding observation coverage radius. If
  \begin{equation}\label{eq:finetuning-consistency-condition}
    \Lcal_{Q_k,P_k}^{\Hcal_k,\star}+\eta_k+h_k\xrightarrow{k\to\infty}0,
  \end{equation}
  then
  \begin{equation}\label{eq:finetuning-prediction-consistency}
    0\le\Lcal_{Q_k,P_k}^\star(\hat e_k)\le\Lcal_{Q_k,P_k}(\hat b_k;\hat e_k)\xrightarrow{k\to\infty}0.
  \end{equation}
\end{corollary}

\begin{proof}
  For every pair in $\Hcal$, the empirical error is no larger than its full-family error. Hence the fine-tuning rule gives
  \begin{equation}
    \widehat{\Lcal}_{Q,P,n}(\hat b;\hat e)\le\inf_{(e,b)\in\Hcal}\widehat{\Lcal}_{Q,P,n}(b;e)+\eta\le\Lcal_{Q,P}^{\Hcal,\star}+\eta.
  \end{equation}
  Apply~\cref{lem:record-extension} to the returned pair to obtain the upper bounds. The lower bound follows because $(\hat e,\hat b)\in\Hcal$.

  Under~\cref{eq:finetuning-consistency-condition}, all three nonnegative terms tend to zero. In particular, $h_k\to0$, so $\rho(h_k)\to0$. The same bound gives $\Lcal_{Q_k,P_k}(\hat b_k;\hat e_k)\to0$, and the optimal prediction error is no larger.
\end{proof}

\subsection{General Case}

\begin{theorem}[General Fine-Tuning Bound]\label{thm:general-finetuning-recovery}
  Every $\eta$-optimal fine-tuning output satisfying~\cref{ass:uniform-prediction-stability} obeys
  \begin{equation}\label{eq:general-finetuning-recovery}
    \Lcal_{Q,P}^\star(\hat e)\le\boxed{\Rcal_Q^\star(\hat e)}\le \omega_P\!\Bigl(2\bigl[\Lcal_{Q,P}^{\Hcal,\star}+\eta +\rho(h_n)+h_n\bigr]\Bigr).
  \end{equation}
  For stages $k\ge1$, allow the training states, query sets, model classes and learned state spaces to vary, with a common modulus $\rho$ for the returned outputs. Let $(\hat e_k,\hat b_k)$ denote the $\eta_k$-optimal outputs. If $(P_k)_{k\ge1}$ is a query sequence separating states in $\Fcal$ and
  \begin{equation}
    \Lcal_{Q_k,P_k}^{\Hcal_k,\star}+\eta_k+h_k\xrightarrow{k\to\infty}0,
  \end{equation}
  then $\Rcal_{Q_k}^\star(\hat e_k)\to0$. Here $h_k$ is the coverage radius of the training observations on $Q_k\cup P_k$ at stage $k$, defined as in~\cref{eq:training-coverage}.
\end{theorem}

\begin{proof}
  Substitute~\cref{eq:finetuning-prediction-bound} into~\cref{thm:paper-general-bounds}, using $\Lcal_{Q,P}^\star(\hat e)\le\Lcal_{Q,P}(\hat b;\hat e)$ and the monotonicity of $\omega_P$.

  For the asymptotic claim, \cref{cor:finetuning-prediction} gives vanishing optimal prediction error. State separation and~\cref{thm:paper-general-equivalence} then give vanishing recovery error.
\end{proof}

\subsection{Low-Dimensional Case}

Using the targets constructed in~\cref{thm:paper-low-bounds} makes the query budget explicit under the same training rule and stability condition.

\begin{theorem}[Low-Dimensional Fine-Tuning Bound]\label{thm:low-dimensional-finetuning-recovery}
  Suppose $(\Fcal,Q)$ satisfies~\cref{ass:low-dimensional-manifold} with constants $C,d,\alpha$. For every integer $m\ge C^2/2$, choose $P\subseteq\Ucal\setminus Q$ as in~\cref{thm:paper-low-bounds}, with $|P|\le m$. Every $\eta$-optimal fine-tuning output satisfying~\cref{ass:uniform-prediction-stability} obeys
  \begin{equation}\label{eq:low-dimensional-finetuning-recovery}
    \Lcal_{Q,P}^\star(\hat e)\le\boxed{\Rcal_Q^\star(\hat e)}\le \frac{2}{\alpha}\bigl[\Lcal_{Q,P}^{\Hcal,\star}+\eta +\rho(h_n)+h_n\bigr] +\frac{2+\alpha}{\alpha}\left(\frac{C^2}{2m}\right)^{\frac{1}{2d}}.
  \end{equation}
  For budgets $m\to\infty$, allow the training states, input sets, model classes, and learned state spaces to vary, selecting each $P_m$ by the same construction. Let $(\hat e_m,\hat b_m)$ denote the $\eta_m$-optimal outputs. If $C,d,\alpha$ and $\rho$ are common across budgets and
  \begin{equation}
    \Lcal_{Q_m,P_m}^{\Hcal_m,\star}+\eta_m+h_m\xrightarrow{m\to\infty}0,
  \end{equation}
  then $\Rcal_{Q_m}^\star(\hat e_m)\to0$. Here $h_m$ is the coverage radius of the training observations on $Q_m\cup P_m$ at budget $m$, defined as in~\cref{eq:training-coverage}.
\end{theorem}

\begin{proof}
  Select $P$ by~\cref{thm:paper-low-bounds}, before the training class and output are specified. Substitute~\cref{eq:finetuning-prediction-bound} into that theorem, using $\Lcal_{Q,P}^\star(\hat e)\le\Lcal_{Q,P}(\hat b;\hat e)$.

  Under the stated convergence condition, \cref{cor:finetuning-prediction} gives vanishing optimal prediction error. Applying~\cref{thm:paper-low-equivalence} to the selected target sets gives vanishing recovery error.
\end{proof}

The bounds identify the contributions of the model class, optimization tolerance, training coverage, and prediction stability. The low-dimensional result also makes the target-query budget explicit: $n$ counts training records, while $m$ counts target queries.

\section{Recovery from Noisy Observations}\label{sec:noisy-recovery}

The construction in~\cref{thm:paper-low-bounds} also applies when a query's response is a conditional mean and its observation is noisy. The state specifies the mean response profile; the measurement law describes how a recording is drawn from it. Prediction error and measurement perturbation then enter the same distance to a covering center. Restricting the construction to a frequently encountered state family gives a probability guarantee and an integrated recovery bound.

Let $\Fcal$ be compact with diameter $D$, and let $G\subseteq\Fcal$ be nonempty and compact. Suppose $(G,Q)$ satisfies~\cref{ass:low-dimensional-manifold} with constants $C_G,d,\alpha$. For an integer $m\ge C_G^2/2$, select $P$ and covering centers by~\cref{thm:paper-low-bounds} on $G$.

Let $\Phi=b\circ e$ be Borel measurable and $K$-Lipschitz on the full input space, with prediction error at most $\ell$ on $G$. Decode to the covering center closest to $b(z)$ on $P$, resolving ties by index. This defines a \textit{state decoder} $a(z)$; the composition $a\circ e$ is measurable because it is a finite selection determined by $\Phi$.

\begin{samepage}
  \begin{lemma}[Recovery under Perturbation]\label{lem:perturbed-recovery}
    Under this setup, if $f\in G$ and $d_Q(x,f(Q))\le\eta$, then
    \begin{equation}\label{eq:perturbed-recovery}
      d_{\Ucal}((a\circ e)(x),f)\le\frac2\alpha(\ell+K\eta) +\frac{2+\alpha}{\alpha}\left(\frac{C_G^2}{2m}\right)^{1/(2d)}.
    \end{equation}
  \end{lemma}
\end{samepage}

\begin{proof}
  The triangle inequality and input sensitivity give
  \begin{equation}
    d_P(\Phi(x),f(P))\le d_P(\Phi(x),\Phi(f(Q))) +d_P(\Phi(f(Q)),f(P)) \le K\eta+\ell.
  \end{equation}
  Apply~\cref{lem:cover-decoding} on $G$ with $y=\Phi(x)$ and the radius and allowance fixed in~\cref{thm:paper-low-bounds}. Substitution gives~\cref{eq:perturbed-recovery}.
\end{proof}

The pointwise bound becomes a probability guarantee when both the state and the measurement lie in their specified ranges.

\begin{corollary}[Recovery under Noise]\label{cor:noisy-subfamily-recovery}
  Under the preceding setup, let $F$ have a Borel law $\mu$ with $\mu(G)\ge1-\delta$, where $\delta\in[0,1)$. Suppose the measured input $X$ satisfies, for every $f\in G$,
  \begin{equation}
    \Pr\bigl(d_Q(X,f(Q))\le\eta\mid F=f\bigr)\ge1-\zeta,
  \end{equation}
  where $\eta\ge0$ and $\zeta\in(0,1)$. With
  \begin{equation}
    B\coloneqq\min\left\{D,\frac2\alpha(\ell+K\eta)+\frac{2+\alpha}{\alpha}\left(\frac{C_G^2}{2m}\right)^{1/(2d)}\right\},
  \end{equation}
  the recovery error obeys
  \begin{align}\label{eq:noisy-subfamily-recovery}
    \Pr\bigl(d_{\Ucal}((a\circ e)(X),F)\le B\bigr)&\ge1-\delta-\zeta,\\
    \mathbb E\,d_{\Ucal}((a\circ e)(X),F)^2&\le B^2+D^2(\delta+\zeta).
  \end{align}
\end{corollary}

\begin{proof}
  The event $F\in G$ and $d_Q(X,F(Q))\le\eta$ has probability at least
  \begin{equation}
    \mu(G)(1-\zeta)\ge(1-\delta)(1-\zeta)\ge1-\delta-\zeta.
  \end{equation}
  On this event, \cref{lem:perturbed-recovery} and the diameter bound give error at most $B$. Off it, the error is at most $D$ because both states lie in $\Fcal$. This proves the probability bound; splitting the squared-error expectation over the two events gives the second bound.
\end{proof}

\mypara{Count observations.}
For a conditional Poisson measurement law, take $X_q=N_q/T$, where $N_q\mid F=f$ has mean $Tf(q)$ and $0\le f(q)\le\Lambda$ on $G$. Here $T$ is the exposure. A valid input-noise radius is
\begin{equation}\label{eq:poisson-input-radius}
  \eta(T,\zeta)=\sqrt{\frac{2\Lambda\log(2|Q|/\zeta)}{T}}+\frac{2\log(2|Q|/\zeta)}{3T}.
\end{equation}
Indeed, the centered Poisson moment-generating function obeys $\log\mathbb E e^{u(N_q-Tf(q))}=Tf(q)(e^u-1-u)$. The bound $e^u-1-u\le u^2/[2(1-u/3)]$ for $0\le u<3$, and the corresponding lower-tail bound, give
\begin{equation}
  \Pr\left(\left|N_q-Tf(q)\right|>\sqrt{2T\Lambda v}+\frac{2v}{3}\,\middle|\,F=f\right)\le2e^{-v}.
\end{equation}
Choose $v=\log(2|Q|/\zeta)$ and take a union bound. This uses each channel's conditional Poisson law and permits dependence between channels. Exposure controls the measurement contribution to recovery, through $K\eta(T,\zeta)$.

\mypara{Prediction sensitivity.}
For a continuously differentiable \textit{predictor} composition, a sufficient full-input constant is
\begin{equation}
  K\ge\sup_{x\in\Rbb^Q}\|J_\Phi(x)\|_{\infty\to\infty}=\sup_{x\in\Rbb^Q}\max_{p\in P}\sum_{q\in Q}\left|\frac{\partial\Phi_p(x)}{\partial x_q}\right|.
\end{equation}
The derivative bound may instead be taken on a convex domain containing every input and intervening segment used in the guarantee.

\mypara{Functional observability.}
Functional observers recover selected state coordinates in structured dynamical networks; algebraic and graph conditions characterize which targets are accessible~\citep{Montanari2022Functional,Montanari2023Target}. Observability also guides the eigenstructure and sparsity of neural state-space architectures~\citep{Gracyk2025Observability,Hamdan2024SparseMamba}. Here, the state is specified by its complete query-response profile, and quantitative target observability compares its full metric with the remaining responses. The query-budget construction uses this comparison to turn target prediction into a state-valued recovery map. Retracing adds cycle consistency to predictive state learning in reinforcement learning~\citep{Yu2022Retracing}; our bounds concern the distinctions preserved by a fixed \textit{encoder} and the error of recovering its response profile.

\makeatletter
\setlength{\@fptop}{0pt}
\setlength{\@fpsep}{12pt}
\setlength{\@fpbot}{0pt plus 1fil}
\setlength{\@dblfptop}{0pt}
\setlength{\@dblfpsep}{12pt}
\setlength{\@dblfpbot}{0pt plus 1fil}
\makeatother

\section{Synthetic Experiment Details}\label{sec:synthetic-details}

The generator specifies the full state family and every query response. This makes recovery error, target separation, and the constants in the error bound directly computable. We describe the generator and inverse map first, then examine the budget dependence and the assumptions responsible for it.

The coordinatewise construction is followed by a second, joint-query construction in~\cref{sec:synthetic-mixed}, which supplies the calibration records for a global recovery certificate.

\subsection{State Family and Query}

The GT observation law is a deterministic function of a sampled latent vector $x$. For \textit{Sphere}, $x$ is uniform on the unit sphere in $\Rbb^{d+1}$; for \textit{Torus}, its coordinates are cosine--sine pairs of independent uniform angles. For \textit{General}, $x$ ranges over the box $[-1,1]^D$ with independent uniform coordinates. A fixed signed coordinate-repetition map expands the latent vector to the common ambient dimension. The observation law is given in~\cref{eq:synthetic-gt-law}.
The coefficients are sampled once and shared across families and repetitions. The map $i(j)$ cycles through the latent coordinates; signs $s_j$ are fixed random signs for the low-dimensional families and positive for the general family. Each coordinate has input and target channels. Inputs contain every input channel, while target sets are nested prefixes of a fixed random permutation of the target pool. The query coefficients are public metadata supplied to the \textit{encoder} and \textit{predictor}. These are structured, separable observation laws; their geometry and full observation profiles can be evaluated exactly.

\begin{table}[htbp]
  \centering
  \caption{Synthetic design. The circle is included once, as the sphere of intrinsic dimension one. The general family has independent ambient coordinates.}
  \label{tab:synthetic-design}
  \small
  \begin{tabular*}{\columnwidth}{@{\extracolsep{\fill}}lr@{}}
    \toprule
    Design & Value \\
    \midrule
    Ambient coordinates $D$ & $16{,}384$ \\
    Channels per coordinate & $4$ ($2$ input, $2$ target) \\
    \textit{Sphere} dimension $d$ & $1,2,4,8,16,32$ \\
    \textit{Torus} dimension $d$ & $2,4,8,16,32$ \\
    Coefficient $a_{jk}$ & $\mathrm{Unif}[.5,1.5)$ \\
    Offset $c_{jk}$ & $\mathrm{Unif}[-.5,.5)$ \\
    Train / validation / test states & $4{,}096/256/256$ \\
    Paired repetitions & $3$ \\
    Budgets $m$ & $2^7,2^9,2^{11},2^{13},2^{15}$ \\
    Configurations & $180$ \\
    \bottomrule
  \end{tabular*}
\end{table}

\begin{table*}[!htbp]
  \centering
  \caption{Random-repetition design. Reported dispersions are sample standard deviations across repetitions of the same condition, with the biological aggregation held fixed. A repetition changes the sources of randomness specified by its protocol.}
  \label{tab:repetition-design}
  \small
  \setlength{\tabcolsep}{2pt}
  \begin{tabular*}{\textwidth}{@{\extracolsep{\fill}}lp{48pt}cp{151pt}@{}}
    \toprule
    Experiment & Random seeds & Runs per case & Pairing / audit \\
    \midrule
    Synthetic sweep / domain audit & $19,29,39$ & $3$ & \textit{General} first repetition: $2^{20}$ grid; others $2^{19}$ \\
    HC-3 prediction / matched recovery & $19,29,39$ & $3$ & Paired partitions, initializations and held-out blocks \\
    \bottomrule
  \end{tabular*}
\end{table*}

Throughout the paper, decimals are rounded to at most four places and omit the leading zero. Exact zeros in decimal summaries appear as $.0000$; nonzero values rounding to zero appear as $.0001$ in magnitude, retaining their sign. All calculations use unrounded values.

\subsection{Prediction and Deterministic Recovery}

A shared coordinatewise MLP encodes input observations and their query coefficients. A second shared MLP combines this learned state with each target channel's coefficients to predict its observation. We minimize the mean, over training states, of the maximum absolute target error. Every configuration is evaluated at the final training epoch specified in advance. Matching pilot configurations enter the paired-repetition design once.

\begin{table}[htbp]
  \centering
  \caption{Synthetic training and deterministic recovery settings. Gradient updates train the \textit{encoder} and \textit{predictor}.}
  \label{tab:synthetic-training}
  \small
  \begin{tabular*}{\columnwidth}{@{\extracolsep{\fill}}lr@{}}
    \toprule
    Training & Value \\
    \midrule
    \textit{Encoder} widths & $6\to64\to64\to16$ \\
    \textit{Predictor} widths & $18\to64\to64\to1$ \\
    Hidden activation & GELU \\
    Optimizer / batch & AdamW / $8$ \\
    Epochs / updates & $200$ / $102{,}400$ \\
    Learning rate, epochs $1$--$100$ & $10^{-3}$ \\
    Learning rate, epochs $101$--$200$ & $10^{-4}$ \\
    Oracle bisection rounds & $44$ \\
    \bottomrule
  \end{tabular*}
\end{table}

The ideal GT map in~\cref{eq:synthetic-gt-oracle} minimizes target discrepancy; the implementation approximates it by bisection.
Monotonicity of each observation channel converts an allowed discrepancy into a latent-coordinate interval. We intersect these intervals and bisect the discrepancy until the resulting box intersects the known latent domain. For a sphere, intersection is decided by the minimum and maximum attainable squared norms; for a torus, the same criterion is applied to each circle factor. A deterministic feasible point then defines the recovered state. The inverse uses the predicted responses and the known latent domain; the test state's latent vector enters the evaluation of recovery error.

For each test state $f$, we record prediction error $d_P((b\circ e)(f(Q)),f(P))$ and recovery error $d_{\Ucal}(\widehat f((b\circ e)(f(Q))),f)$. Recovery is evaluated over every channel of the full query universe. Summaries report mean $\pm$ sample standard deviation over paired repetitions. Random seed sets are listed in~\cref{tab:repetition-design}. Bound checks instead use the maximum over the test states within each run. Seed pairing fixes initialization, state samples, and target order across budgets; seeds jointly change these sources of randomness.

\mypara{Training choices and computation.}\label{sec:synthetic-reproducibility}
Dimensions and target counts form geometric grids spanning the three state families. The shared MLP widths, batch size, and sample counts are fixed design parameters; every reported configuration uses the same optimizer and two-stage schedule. Preliminary convergence checks informed the $200$-epoch endpoint, which is held fixed across the full suite. AdamW uses weight decay $.01$, moments $(.9,.999)$, and numerical stabilizer $10^{-8}$. Training uses Float32; numerical inversion and the domain audit use Float64. The seed table and supplied case manifest specify initialization, state generation, target order, and minibatch permutations.

With the network widths fixed, a batch of $B$ states costs $O(B(D+m))$ time and activation memory: the shared \textit{encoder} processes $D$ coordinates and the \textit{predictor} evaluates $m$ targets. Parameters are shared across coordinates and queries. Each numerical inverse uses $44$ feasibility rounds, each linear in the target count and latent-coordinate count; full-profile scoring is linear in $D$. The uniform audit evaluates each target on the grid in~\cref{tab:uniform-audit-settings}; target and grid batching bound working memory without changing the evaluated domain.

\mypara{Reproduction materials and infrastructure.}\label{sec:reproduction-materials}
The accompanying archive \texttt{ontic-reproducibility.zip} supplies source code, case manifests, dependency versions, original numerical endpoints, and commands for preprocessing, training, evaluation, and aggregation. Figure and table sources are included with the manuscript source. Synthetic data are generated by the supplied response laws and seeds. HC-3 recordings are obtained from their provider under the terms in~\cref{sec:hc3-data-access}. The archive retains all endpoints of the included suites and separates within-condition seed variation from biological aggregation.

Experiments ran in a Linux batch-computing environment. The supplied training configuration uses one NVIDIA A100 (80\,GB); observation-geometry calculations use CPU numerical linear algebra. The reproduction environment is Python 3.10.19, PyTorch 2.10.0, NumPy 1.26.4, SciPy 1.14.1, Matplotlib 3.8.4, and mpmath 1.3.0. The archive includes version-pinned dependencies. Checkpoints retain optimizer, sampler, and random-generator state, with separate latest and previous copies. Resumption uses the same device type and PyTorch version.

The main-text observability inequality and constant $\alpha_P$ are given in~\cref{eq:synthetic-main-observability}. We check coordinate coverage and its budget requirement before applying the transfer bound.

\subsection{Assumption Checks}

\mypara{Observability.}
Write $r$ for the number of ambient latent coordinates: $d+1$ for the sphere, $2d$ for the torus, and $D$ for the box. Let $I(P)=\{i(j):(j,k)\in P\}$. For the observation law in~\cref{eq:synthetic-gt-law},
\begin{equation}\label{eq:synthetic-exact-separation}
  P\text{ separates }\Fcal\quad\Longleftrightarrow\quad I(P)=\{1,\ldots,r\}.
\end{equation}
Full coordinate coverage gives separation by~\cref{prop:synthetic-observability}. Conversely, each latent domain is invariant under changing the sign of one coordinate and contains a state with that coordinate nonzero. Reflecting an unobserved coordinate therefore gives two distinct states with identical target responses. Their full profiles differ because the query universe observes every coordinate.

\begin{proof}[Proof of~\cref{prop:synthetic-observability}]
  \emph{Bound each channel.} Fix $f=f_x$ and $g=f_{x'}$ in $\Fcal$. For every channel $(j,k)$ and $|t|\le1$,
  \begin{equation}
    a_{jk}\operatorname{sech}^2(a_{jk}+|c_{jk}|)\le\left|\frac{d}{dt}\tanh(a_{jk}s_jt+c_{jk})\right|\le a_{jk}.
  \end{equation}
  The mean value theorem, followed by maximization over the selected target channels and over all channels, respectively, gives
  \begin{equation}
    d_P(f,g)\ge\max_i\ell_i|x_i-x_i'|,\qquad d_{\Ucal}(f,g)\le\max_i A_i|x_i-x_i'|.
  \end{equation}

  A single increasing channel turns coordinate separation into response separation:
  \begin{center}
    \makebox[169.24588pt][c]{\raisebox{-20.00000pt}[81.46638pt][0.00000pt]{\includegraphics[width=209.24588pt,height=121.46638pt]{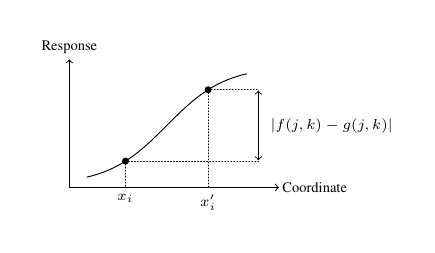}}}

  \end{center}

  \emph{Compare the distances.} Every coordinate has a target channel with positive slope magnitude, so $\ell_i>0$ and $A_i>0$. There are finitely many coordinates, hence $\alpha_P=\min_i\ell_i/A_i>0$. Since $\ell_i\ge\alpha_P A_i$ for every $i$,
  \begin{equation}
    d_P(f,g)\ge\max_i\ell_i|x_i-x_i'| \ge\alpha_P\max_i A_i|x_i-x_i'| \ge\alpha_Pd_{\Ucal}(f,g).
  \end{equation}
  This proves~\cref{eq:synthetic-main-observability}.
\end{proof}

 For each certified target set, we recompute this constant and the test-error transfer. The input and candidate-target pools cover all latent coordinates.

\mypara{Query budget.}
The actual target sampler permits a finite-budget calculation. Let $M$ be the target-pool size, and let $n_i$ count its channels observing latent coordinate $i$. A uniform subset of size $m$ misses coordinate $i$ with probability $\binom{M-n_i}{m}/\binom{M}{m}$. Hence
\begin{equation}\label{eq:synthetic-random-budget}
  \Pr\{P\text{ fails to separate }\Fcal\}\le\min\left\{1,\sum_{i=1}^{r}\frac{\binom{M-n_i}{m}}{\binom{M}{m}}\right\},
\end{equation}
where a binomial coefficient is zero when its lower argument exceeds its upper argument. On the complementary event, the global error transfer holds simultaneously for every \textit{predictor} and every state, using the corresponding $\alpha_P$. In particular, this event is determined before predictive training. Each term in the sum is at most $\exp(-m\min_i n_i/M)$. Thus $m\ge(M/\min_i n_i)\log(r/\eta)$ suffices for failure probability at most $\eta$ whenever this budget is available; querying the full pool always separates. For the repeated-coordinate low-dimensional families, $M/\min_i n_i$ is of the order of $r$, giving a budget of order $r\log(r/\eta)$.

We evaluate the without-replacement bound using the actual channel counts (\cref{tab:synthetic-budget-diagnostic}). The low-dimensional families repeat each latent coordinate across many channels, so their sufficient budgets are much smaller. These are family-specific sampling guarantees; the general covering theorem accommodates arbitrary admissible response laws.

\begin{table}[htbp]
  \centering
  \caption{Smallest budget for which~\cref{eq:synthetic-random-budget} gives at least $95\%$ probability of separation. The pool contains $32{,}768$ queries; $r$ counts latent coordinates, while $d$ is intrinsic dimension. For the largest sphere and torus at budget $512$, failure-probability upper bounds are $4.18\times10^{-6}$ and $.0189$. }
  \label{tab:synthetic-budget-diagnostic}
  \small
  \begin{tabular*}{\columnwidth}{@{\extracolsep{\fill}}lrrr@{}}
    \toprule
    Family & $d$ & $r$ & Sufficient $m$ \\
    \midrule
    \textit{Sphere} & $1$ & $2$ & $6$ \\
    \textit{Sphere} & $2$ & $3$ & $11$ \\
    \textit{Sphere} & $4$ & $5$ & $21$ \\
    \textit{Sphere} & $8$ & $9$ & $45$ \\
    \textit{Sphere} & $16$ & $17$ & $97$ \\
    \textit{Sphere} & $32$ & $33$ & $211$ \\
    \textit{Torus} & $2$ & $4$ & $16$ \\
    \textit{Torus} & $4$ & $8$ & $38$ \\
    \textit{Torus} & $8$ & $16$ & $90$ \\
    \textit{Torus} & $16$ & $32$ & $203$ \\
    \textit{Torus} & $32$ & $64$ & $452$ \\
    \midrule
    \textit{General} & $16,384$ & $16,384$ & $32,711$ \\
    \bottomrule
  \end{tabular*}
\end{table}

\mypara{Covering number.}
The same generator supplies explicit covering constants. Put $A=\max_{jk}a_{jk}$. Angular coordinates for the sphere and torus give, respectively,
\begin{equation}
  \Ncal_{\Fcal}(\varepsilon)\le\left(1+\frac{2\pi d A}{\varepsilon}\right)^d,\qquad \Ncal_{\Fcal}(\varepsilon)\le\left(1+\frac{2\pi A}{\varepsilon}\right)^d.
\end{equation}
Indeed, changing a sphere angle changes its unit vector by at most the angular displacement in Euclidean norm; summing over angles gives the first grid radius. Each sine--cosine factor of a torus is Lipschitz in its own angle, giving the second radius in the maximum-coordinate metric. The observation map is $A$-Lipschitz in that metric. Choosing a nearest grid point in the parameter domain therefore covers every state. For $\varepsilon\le1$, these bounds have the required form $C\varepsilon^{-d}$ on the low-dimensional families. The constants describe a conservative existence construction for these state families. A regular grid on the box gives $\Ncal_{\Fcal}(\varepsilon)\le(1+2A/\varepsilon)^D$.

\subsection{Error Transfer}

When every latent coordinate has a target channel, \cref{eq:synthetic-main-observability} supplies the global observability constant $\alpha_P$.
For the ideal GT map, the true state is a feasible candidate, so the fitted target residual is at most the prediction error. For a numerical output $g\in\Fcal$ whose target residual is within $\xi\ge0$ of the minimum, observability and the triangle inequality give
\begin{equation}\label{eq:numerical-inverse-transfer}
  \alpha_Pd_{\Ucal}(g,f)\le d_P(g(P),y)+d_P(y,f(P)) \le2d_P(y,f(P))+\xi.
\end{equation}
The exact map has $\xi=0$. With exact feasibility tests and feasible-point construction, $44$ bisections of $[0,\|y\|_\infty+1]$, followed by a $10^{-10}$ feasibility expansion, give $\xi\le2^{-44}(\|y\|_\infty+1)+10^{-10}$. For a feasible Float64 output, $\xi$ also includes arithmetic effects on the target residual. Finite-test plots compare the numerical recovery errors with the exact-map benchmark $2\widehat L/\alpha_P$, where $\widehat L$ is the maximum test prediction error. Coordinate coverage gives the separation test in~\cref{eq:synthetic-exact-separation}.

\begin{table}[htbp]
  \centering
  \caption{Numerical recovery compared with the exact-map benchmark on separating target sets. All observed errors lie below the benchmark, computed from maximum test prediction errors and global derivative bounds.}
  \label{tab:synthetic-bound-check}
  \small
  \begin{tabular*}{\columnwidth}{@{\extracolsep{\fill}}lr@{}}
    \toprule
    Quantity & Value \\
    \midrule
    Certified target sets & $162$ \\
    Finite-test recovery bound, range & $.023$--$.274$ \\
    Observed / bound ratio, median & $.075$ \\
    Observed / bound ratio, range & $.020$--$.247$ \\
    \bottomrule
  \end{tabular*}
\end{table}

\subsection{Uniform Prediction and Recovery}\label{sec:synthetic-uniform}

The preceding transfer checks use test states. The coordinatewise architecture also permits a bound over the entire known family. For a target $t=(j,k)$, put $v=s_jx_{i(j)}$ and write $h_t(v)$ for the predicted observation minus $\tanh(a_{jk}v+c_{jk})$. Writing $e_{\rm sh}$ and $b_{\rm sh}$ for the shared MLPs in~\cref{tab:synthetic-training}, the implemented coordinatewise composition is
\begin{equation}
  \begin{aligned}
    h_t(v)={}&b_{\rm sh}\!\Bigl(e_{\rm sh}\bigl(\tanh(a_{j1}v+c_{j1}), \tanh(a_{j2}v+c_{j2}),a_{j1},a_{j2},c_{j1},c_{j2}\bigr), a_{jk},c_{jk}\Bigr)\\
      &-\tanh(a_{jk}v+c_{jk}).
  \end{aligned}
\end{equation}
Thus each target prediction depends only on $v$ and its query coefficients. Each latent coordinate ranges over $[-1,1]$ in every family, so
\begin{equation}\label{eq:synthetic-scalar-uniform}
  \Lcal_{Q,P}(b;e)=\max_{t\in P}\sup_{v\in[-1,1]}|h_t(v)|.
\end{equation}
The supremum commutes with the finite maximum: each target's extremum can be attained at a state of its own. Agreement between the network forward pass and its scalar form checks this reduction before the domain audit.

We bound prediction sensitivity in the full input-observation norm using the trained weights. Let $W_1,\ldots,W_6$ be the successive \textit{encoder} and \textit{predictor} matrices, $w_1,w_2$ the varying observation columns of $W_1$, and $W_{4,h}$ the hidden-state columns of the first \textit{predictor} matrix. Using the GELU derivative bound $1.13$, a valid constant is
\begin{equation}\label{eq:synthetic-weight-stability}
  K=1.13^4\|W_6\|_2\|W_5\|_2\|W_{4,h}\|_2 \cdot\|W_3\|_2\|W_2\|_2\max_{\sigma\in\{-1,1\}}\|w_1+\sigma w_2\|_2.
\end{equation}
Every varying input observation enters this calculation. The \textit{encoder--predictor} composition therefore satisfies~\cref{ass:uniform-prediction-stability} with $\rho(t)=Kt$. The scalar error has Lipschitz constant $H_t=K\max(a_{j1},a_{j2})+a_{jk}$, since $\tanh$ is Lipschitz with constant one.

On a regular grid $v_\ell=-1+2\ell/N$, let $M_t=\max_{0\le\ell\le N}|h_t(v_\ell)|$. Every point of the interval lies within $1/N$ of the grid. Thus
\begin{equation}\label{eq:synthetic-uniform-enclosure}
  \max_{t\in P}M_t\le\boxed{\Lcal_{Q,P}(b;e)}\le\max_{t\in P}\left(M_t+\frac{H_t}{N}\right).
\end{equation}
For the reported numerical values, let $\widetilde M_t$ be the Float64 grid maximum and use the arithmetic allowance $10^{-6}$. We use
\begin{equation}\label{eq:synthetic-numerical-enclosure}
  \underline L=\max\{0,\max_{t\in P}\widetilde M_t-10^{-6}\},\qquad \overline L=\max_{t\in P}(\widetilde M_t+H_t/N)+10^{-6}.
\end{equation}
The numerical enclosure assumes that $10^{-6}$ covers grid-evaluation and enclosure-arithmetic error, and that the evaluated $H_t$ remain valid Lipschitz upper bounds. Recovery additionally assumes that the evaluated $\alpha_P>0$ remains a valid observability lower bound. Tables report Float64 evaluations of the resulting conditional analytic bounds at the stated display precision; rounding the final quotient and its displayed value is separate from the analytic inequality. The implementation uses high-precision spot checks and ordinary Float64 arithmetic. \Cref{tab:uniform-audit-settings} records the allowance and checks.

The grid is an offline domain audit of the fixed \textit{predictor}. Its resolution is determined by the weight bound and a preset interpolation tolerance. Under the arithmetic conditions above, the ideal GT map in~\cref{eq:synthetic-gt-oracle} obeys
\begin{equation}\label{eq:synthetic-uniform-recovery}
  \sup_{f\in\Fcal}d_{\Ucal}\bigl(\widehat f((b\circ e)(f(Q))),f\bigr)\le\frac{2\overline L}{\alpha_P},
\end{equation}
where $\overline L$ denotes the numerical upper enclosure. For feasible numerical outputs in $\Fcal$, a uniform residual slack $\xi$ adds $\xi/\alpha_P$ by~\cref{eq:numerical-inverse-transfer}.

The audit includes all repetitions for the largest sphere and torus at budgets above their sufficient sampling thresholds, and for the independent family at its full budget. Each bound uses the final training checkpoint and its training target set. All reported ideal-map bounds improve on the full-profile diameter (\cref{tab:synthetic-uniform}).

\begin{table}[htbp]
  \centering
  \caption{Whole-domain audit settings. Bounds use the final training checkpoints and their training target sets. Seed sets and grid sizes appear in~\cref{tab:repetition-design}.}
  \label{tab:uniform-audit-settings}
  \small
  \begin{tabular*}{\columnwidth}{@{\extracolsep{\fill}}lr@{}}
    \toprule
    Setting & Value \\
    \midrule
    \textit{Sphere} and \textit{Torus}: $(d,m)$ & $(32,512)$ \\
    \textit{General}: $(D,m)$ & $(16{,}384,32{,}768)$ \\
    Interpolation tolerance & $.002$ \\
    GELU derivative bound & $1.13$ \\
    Arithmetic allowance & $10^{-6}$ \\
    Full-input sensitivity $K$ & $374.7$--$705.7$ \\
    Full-profile diameter & $1.8101$ \\
    High-precision check & $70$ digits \\
    Float64 discrepancy & $\le4\times10^{-16}$ \\
    \bottomrule
  \end{tabular*}
\end{table}

\section{Joint-Query Calibration}
\label{sec:synthetic-mixed}

This joint-query family makes the mechanism in~\cref{thm:paper-general-bounds} explicit: target separation converts prediction error into recovery of the full state. Its linear inverse link lets records at known latent coordinates control global prediction error directly.

For the calibration certificate below, each query depends on the whole latent configuration. For $x\in\mathbb S^{r-1}$, define $f_x(q)=\tanh(a_q^\top x)$. The input matrix $A_Q$ and base target matrix $A_P$ have independently drawn unit rows. Target sets vary through $A_\lambda=A_P\operatorname{diag}(1,\ldots,1,\lambda)$ with $0\le\lambda\le1$, so every row has Euclidean norm at most one; the query universe contains the input rows and every target variant. This keeps the ontic state family fixed while changing the distinctions available to prediction.

Consider an \textit{encoder} applying the inverse response link and a matrix, followed by a joint \textit{predictor}:
\begin{equation}
  (b\circ e)(f_x(Q))=\tanh\bigl(W_PW_E\operatorname{arctanh}(f_x(Q))\bigr).
\end{equation}
Inputs are clipped to the valid observation range before applying the inverse link. The resulting full-input prediction map is Lipschitz with constant at most $\cosh^2(1)\|W_PW_E\|_{\infty\to\infty}$.

We obtain prediction, observability and recovery bounds from the same calibration records below. Throughout this section, $\sigma_{\min}(A)=\inf_{\|v\|_2=1}\|Av\|_2$, so positivity means full column rank.

\subsection{From Training Residuals to a Global Bound}

Calibration records connect prediction error to recovery in three steps: they bound the global prediction error, certify that the targets distinguish states, and determine a recovery map.

For the joint-query family, set $A_P=A_\lambda$ and $M=W_PW_E$. Let the calibration matrix $X\in\Rbb^{n\times r}$ have full column rank, with $\sigma=\sigma_{\min}(X)>0$. The inverse-link records $U\in\Rbb^{n\times|Q|}$ and $V\in\Rbb^{n\times|P|}$ approximate $XA_Q^\top$ and $XA_P^\top$ with entrywise errors at most $\eta_Q$ and $\eta_P$, respectively. The weights may be fitted to these records.

\begin{lemma}[Calibration Error]\label{lem:calibration-error}
  Under the calibration setup, let $Y\in\Rbb^{n\times k}$ and $B\in\Rbb^{r\times k}$ satisfy $|(Y-XB)_{ij}|\le\eta$. Then
  \begin{equation}\label{eq:calibration-error-lift}
    \max_j\|(X^\dagger Y-B)_{\cdot j}\|_2\le\frac{\sqrt n\eta}{\sigma},\qquad \|X^\dagger Y-B\|_{2\to2}\le\frac{\sqrt{nk}\eta}{\sigma}.
  \end{equation}
\end{lemma}

\begin{proof}
  Since $X^\dagger X=I$,
  \begin{equation}
    X^\dagger Y-B=X^\dagger(Y-XB).
  \end{equation}
  Each column of $Y-XB$ has norm at most $\sqrt n\eta$, and its Frobenius norm is at most $\sqrt{nk}\eta$. Multiplication by $X^\dagger$, whose operator norm is $1/\sigma$, gives both bounds.
\end{proof}

\begin{lemma}[Calibrated Prediction]\label{lem:calibration-frame}
  Define the residual and its global error bound by
  \begin{align}
    R&\coloneqq UM^\top-V,\\
    \Delta&\coloneqq\max_{p\in P}\|X^\dagger R_{\cdot p}\|_2+\frac{\sqrt n}{\sigma}\bigl(\|M\|_{\infty\to\infty}\eta_Q+\eta_P\bigr).
  \end{align}
  Then
  \begin{equation}\label{eq:frame-prediction-certificate}
    \Lcal_{Q,P}(b;e)\le\max_{p\in P}\|(MA_Q-A_P)_p\|_2\le\Delta.
  \end{equation}
  For exact records, $\eta_Q=\eta_P=0$, the second inequality is an equality.
\end{lemma}

\begin{proof}
  The residual satisfies
  \begin{equation}
    R-X(MA_Q-A_P)^\top=(U-XA_Q^\top)M^\top -(V-XA_P^\top).
  \end{equation}
  Its entries have magnitude at most $\|M\|_{\infty\to\infty}\eta_Q+\eta_P$. Apply~\cref{lem:calibration-error} with $Y=R$ and $B=(MA_Q-A_P)^\top$; the column bound and triangle inequality give the bound $\Delta$ on every row. The $1$-Lipschitz link $\tanh$ gives the prediction bound. With exact records, $R=X(MA_Q-A_P)^\top$, so $X^\dagger R=(MA_Q-A_P)^\top$ and the second inequality is an equality.
\end{proof}

The same records estimate the target matrix. Its smallest singular value measures whether the targets retain every latent direction.

\begin{lemma}[Calibrated Observability]\label{lem:calibrated-observability}
  Define
  \begin{align}
    \widehat A_P&\coloneqq(X^\dagger V)^\top,\qquad \tau\coloneqq\frac{\sqrt{n|P|}\eta_P}{\sigma},\\
    \widehat\alpha_P&\coloneqq\frac{\operatorname{sech}^2(1)}{\sqrt{|P|}}\bigl(\sigma_{\min}(\widehat A_P)-\tau\bigr)_+.
  \end{align}
  Then $\|\widehat A_P-A_P\|_{2\to2}\le\tau$ and
  \begin{equation}\label{eq:frame-observability-certificate}
    d_P(f,g)\ge\widehat\alpha_Pd_{\Ucal}(f,g),\qquad f,g\in\Fcal.
  \end{equation}
  Thus $\widehat\alpha_P>0$ certifies separation. For exact target records, $\widehat A_P=A_P$.
\end{lemma}

\begin{proof}
  Apply~\cref{lem:calibration-error} with $Y=V$, $B=A_P^\top$ and $\eta=\eta_P$. The operator-norm bound gives $\|\widehat A_P-A_P\|_{2\to2}\le\tau$, hence
  \begin{equation}
    \sigma_{\min}(A_P)\ge\bigl(\sigma_{\min}(\widehat A_P)-\tau\bigr)_+.
  \end{equation}
  All query rows have norm at most one, so $d_{\Ucal}(f_x,f_{x'})\le\|x-x'\|_2$ for unit $x,x'$. The derivative of $\tanh$ on $[-1,1]$ is at least $\operatorname{sech}^2(1)$; hence
  \begin{equation}
    \begin{aligned}
      d_P(f_x,f_{x'})&\ge\operatorname{sech}^2(1)\|A_P(x-x')\|_\infty\\
      &\ge\frac{\operatorname{sech}^2(1)\sigma_{\min}(A_P)}{\sqrt{|P|}}\|x-x'\|_2 \ge\widehat\alpha_Pd_{\Ucal}(f_x,f_{x'}).
    \end{aligned}
  \end{equation}
  This proves~\cref{eq:frame-observability-certificate}. For exact target records, $V=XA_P^\top$ gives $\widehat A_P=A_P$.
\end{proof}

When this certificate is positive, the calibrated inverse estimates the latent vector. Normalization returns that estimate to the sphere.

\begin{proposition}[Calibrated Recovery]\label{prop:calibrated-recovery}
  Use $\Delta$, $\widehat A_P$ and $\tau$ from~\cref{lem:calibration-frame,lem:calibrated-observability}, and assume $\sigma_{\min}(\widehat A_P)>\tau$. Let $\nu$ normalize nonzero vectors to the sphere and send zero to the first coordinate vector. Define
  \begin{align}
    c(y)&\coloneqq\operatorname{arctanh}\!\left(\operatorname{clip}_{[-\tanh(1),\tanh(1)]}(y)\right),\\
    \widehat x(y)&\coloneqq\nu\!\left(\widehat A_P^\dagger c(y)\right).
  \end{align}
  This readout satisfies
  \begin{equation}\label{eq:frame-recovery-certificate}
    \sup_{x\in\mathbb S^{r-1}}d_{\Ucal}\bigl(f_{\widehat x((b\circ e)(f_x(Q)))},f_x\bigr)\le\frac{2\bigl(\sqrt{|P|}\Delta+\tau\bigr)}{\sigma_{\min}(\widehat A_P)}.
  \end{equation}
\end{proposition}

\begin{proof}
  Fix a unit $x$ and write $y=(b\circ e)(f_x(Q))=\tanh(MA_Qx)$. Since $c(y)=\operatorname{clip}_{[-1,1]}(MA_Qx)$ and $A_Px\in[-1,1]^{|P|}$, clipping and~\cref{lem:calibration-frame,lem:calibrated-observability} give
  \begin{equation}
    \|c(y)-\widehat A_Px\|_2\le\|c(y)-A_Px\|_2 +\|(A_P-\widehat A_P)x\|_2 \le\sqrt{|P|}\Delta+\tau.
  \end{equation}
  For a unit $x$, normalization is a nearest-point projection onto the sphere, including at zero. Thus, for every vector $v$,
  \begin{equation}\label{eq:sphere-projection-error}
    \|\nu(v)-x\|_2\le\|\nu(v)-v\|_2+\|v-x\|_2\le2\|v-x\|_2.
  \end{equation}
  \begin{center}
    \makebox[85.44550pt][c]{\raisebox{-20.00000pt}[64.65695pt][0.00000pt]{\includegraphics[width=125.44550pt,height=104.65695pt]{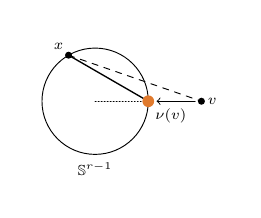}}}

  \end{center}
  Since $\widehat A_P$ has full column rank and all query rows have norm at most one,
  \begin{equation}
    d_{\Ucal}(f_{\widehat x(y)},f_x)\le2\|\widehat A_P^\dagger c(y)-x\|_2 \le\frac{2\|c(y)-\widehat A_Px\|_2}{\sigma_{\min}(\widehat A_P)} \le\frac{2\bigl(\sqrt{|P|}\Delta+\tau\bigr)}{\sigma_{\min}(\widehat A_P)}.
  \end{equation}
  Taking the supremum over $x$ proves~\cref{eq:frame-recovery-certificate}.
\end{proof}

\par\raggedbottom
\section{Observation Geometry}\label{sec:hc3-geometry-details}

We examine how many representative observations cover a later diagnostic block, using the same geometry protocol and time-shift controls. Covariance statistics provide a complementary description. Prediction and model-state recovery use this cohort in~\cref{sec:hc3-prediction-details,sec:hc3-local-details}.

\mypara{Data access and reuse.}\label{sec:hc3-data-access}
The recordings were contributed by Mizuseki, Sirota, Pastalkova, Diba, and Buzs\'aki and distributed as HC-3 through CRCNS.org~\citep{Mizuseki2013HC3,mizuseki2014neurosharing}. The provider makes the data available for scientific use under the \href{https://crcns.org/terms}{CRCNS data-use terms}, which require attribution and citation and restrict redistribution. The \href{https://crcns.org/data-sets/hc/hc-3/about-hc-3}{HC-3 dataset page} identifies the source files and access procedure. Our analysis reuses these existing animal recordings under that published provision for scientific use. The reproduction archive supplies the session manifest and preprocessing code; readers obtain recordings directly from the provider. Original code and aggregate outputs are supplied for confidential peer-review reproduction, and dependencies retain their upstream licenses, listed in the archive.

\subsection{Session and Observation}
The HC-3 cohort~\citep{mizuseki2014neurosharing} spans open-field, linear-track, and wheel recordings (\cref{tab:hc3-panel}). We select sessions using metadata to balance data availability with breadth across recording groups. Eligible sessions have aligned position, recording metadata, and spike files. We allocate sessions across recording prefixes, cycling through the groups within each prefix. Within a group, priority follows processing availability, matched shanks, estimated unit counts, and recording length. This procedure supplies varied recording contexts for comparing observation geometry.

Every included session meets the unit-quality criteria in~\cref{tab:hc3-preprocessing}. Electrode metadata identifies recording prefixes belonging to the same animal. All analyses use this fixed cohort; detailed session statistics appear in~\cref{tab:hc3-geometry-sessions}.

\begin{table*}[!t]
  \centering
  \caption{HC-3 study cohort relative to the complete local catalog, with session counts by animal. The ec015 recording prefix belongs to ec014.}
  \label{tab:hc3-panel}
  \small
  \renewcommand{\arraystretch}{1.1}
  \begin{minipage}[t]{0.51\textwidth}
    \centering
    \begin{tabular*}{\linewidth}{@{\extracolsep{\fill}}lrr@{}}
      \toprule
      Coverage & Catalog & Study \\
      \midrule
      Sessions & 442 & 25 \\
      Animals & 11 & 4 \\
      Recording groups & 119 & 23 \\
      Behavioral contexts & 14 & 5 \\
      \bottomrule
    \end{tabular*}
  \end{minipage}\hfill
  \begin{minipage}[t]{0.42\textwidth}
    \centering
    \begin{tabular*}{\linewidth}{@{\extracolsep{\fill}}lr@{}}
      \toprule
      Animal & Sessions \\
      \midrule
      ec012 & 1 \\
      ec013 & 7 \\
      ec014 & 10 \\
      ec016 & 7 \\
      \bottomrule
    \end{tabular*}
  \end{minipage}
\end{table*}

Each query specifies a unit's spike count over a trailing window. At each sampled observation time, the window ends at that time; evaluating the fixed collection of queries gives the corresponding observation vector. Each query returns the count, logarithmically transformed and scaled using the earlier training block (\cref{tab:hc3-preprocessing}). This normalization stays fixed throughout the recording. Each window lies fully inside its assigned block, with gaps at the boundaries. We derive geometric coordinates from a reference block, test their transfer to a later diagnostic block, and reserve the final block for subsequent experiments. Shorter and longer counting windows test sensitivity to temporal resolution. The calculations operate directly on recorded observations.

\begin{table}[htbp]
  \centering
  \caption{Observation construction and time blocks. Block percentages refer to recording duration.}
  \label{tab:hc3-preprocessing}
  \small
  \setlength{\tabcolsep}{4pt}
  \renewcommand{\arraystretch}{1.1}
  \begin{tabular}{@{}p{0.34\columnwidth}p{0.61\columnwidth}@{}}
    \toprule
    Preprocessing & Setting \\
    \midrule
    File eligibility & XML, LFP, position; $\ge4$ matched shanks, $\ge10$ estimated units, $\ge1{,}000$ position samples \\
    Unit quality & Training rate $\ge.1$\,Hz; $\ge12$ units per session \\
    Retained units & Median 51; range 20--104 \\
    Count window & 1\,s; sensitivity .5 and 2\,s \\
    Window stride & .25\,s \\
    Transform & $\log(1+\text{count})$ \\
    Scale & Per-unit training 95th percentile, floored at 1 \\
    \addlinespace[3pt]
    Training block & Initial 35\% \\
    Reference block & 45--65\% \\
    Diagnostic block & 65--75\% \\
    Reserved block & 75--100\% \\
    Boundary gap & 2\,s \\
    Shank quality & Invalid shank 13 in ec014.723 omitted \\
    \bottomrule
  \end{tabular}
\end{table}

\subsection{Covering Geometry}\label{sec:hc3-covering-details}

For each recording block, let $X$ be the set of sampled observation vectors, equipped with the maximum-coordinate metric. A farthest-first traversal constructs a greedy cover with size $\widehat N_X(\varepsilon)$. The same traversal provides a packing lower bound $\underline N_X(\varepsilon)$, giving $\underline N_X(\varepsilon)\le\Ncal_X(\varepsilon)\le\widehat N_X(\varepsilon)$ with centers in $X$. The lower bound counts selected observations separated by more than $2\varepsilon$. We evaluate the full radius grid, nested sample sizes, and time-shift replicates specified in~\cref{tab:hc3-cover-protocol}.

The reference block fixes the distance scale. Real and shifted diagnostic observations use identical sample indices, sample sizes, and radii. The main comparison uses the grid radius nearest a prespecified fraction of the reference diameter. Curve bands show session interquartile ranges; control bars span the time-shift controls.

\begin{table}[htbp]
  \centering
  \caption{Covering protocol, using the observations and chronological blocks in~\cref{tab:hc3-preprocessing}.}
  \label{tab:hc3-cover-protocol}
  \small
  \begin{tabular}{@{}p{.40\columnwidth}p{.55\columnwidth}@{}}
    \toprule
    Setting & Value \\
    \midrule
    Distance & Maximum absolute coordinate difference \\
    Samples per block & Nested $128$, $256$, $512$, $1024$, limited by block length \\
    Radius grid & $25$ logarithmic scales, $.03$--$1$ reference diameter \\
    Fixed-radius rule & Nearest grid radius to $.25$ reference diameter \\
    Actual fixed radius & $.2685$ reference diameter \\
    Time-shift controls & $19$ replicates per block and window \\
    \bottomrule
  \end{tabular}
\end{table}
\begin{table}[htbp]
  \centering
  \caption{Fixed-radius cover-size comparison over the full cohort. Each session has a smaller cover than its median time-shift control at every window.}
  \label{tab:hc3-cover-window}
  \small
  \begin{tabular}{@{}lrr@{}}
    \toprule
    Window & Cover & Reduction \\
    \midrule
    .5\,s & 475 & 11.2\% \\
    1\,s & 243 & 15.2\% \\
    2\,s & 170 & 14.4\% \\
    \bottomrule
  \end{tabular}
\end{table}
The fixed-resolution comparison shows that recorded population observations admit smaller constructed covers than time-shift controls. The animal and session summaries use the same radius and retain the entire cohort. These are covering statistics for the recorded observations in the specified maximum-coordinate metric.

\begin{table*}[htbp]
  \centering
  \caption{Complete primary-window covering results at the fixed radius. The ratio divides the real greedy-cover size by its median time-shift cover size.}
  \label{tab:hc3-cover-sessions}
  \small
  \begin{tabular*}{\textwidth}{@{\extracolsep{\fill}}lrrrr@{}}
    \toprule
    Session & Units & Samples & Cover & Fixed-Radius Ratio \\
    \midrule
    ec012ec.187 & 20 & 419 & 159 & .874 \\
    ec013.156 & 56 & 389 & 202 & .957 \\
    ec013.198 & 61 & 605 & 357 & .944 \\
    ec013.250 & 37 & 721 & 158 & .963 \\
    ec013.375 & 38 & 683 & 486 & .970 \\
    ec013.406 & 47 & 652 & 413 & .967 \\
    ec013.553 & 52 & 981 & 239 & .895 \\
    ec013.646 & 37 & 510 & 284 & .959 \\
    ec014.123 & 57 & 1024 & 94 & .832 \\
    ec014.183 & 104 & 1024 & 380 & .735 \\
    ec014.195 & 91 & 1024 & 458 & .798 \\
    ec014.215 & 87 & 1024 & 554 & .843 \\
    ec014.260 & 59 & 1024 & 339 & .835 \\
    ec014.277 & 69 & 1024 & 392 & .800 \\
    ec014.723 & 35 & 1024 & 125 & .796 \\
    ec015.043 & 48 & 1024 & 127 & .804 \\
    ec015.045 & 47 & 1024 & 325 & .819 \\
    ec015.047 & 46 & 1024 & 243 & .877 \\
    ec016.233 & 52 & 751 & 155 & .657 \\
    ec016.269 & 57 & 1024 & 75 & .765 \\
    ec016.397 & 51 & 1024 & 195 & .848 \\
    ec016.430 & 44 & 1024 & 191 & .923 \\
    ec016.448 & 43 & 1024 & 302 & .915 \\
    ec016.483 & 40 & 1024 & 119 & .804 \\
    ec016.582 & 61 & 1024 & 470 & .902 \\
    \bottomrule
  \end{tabular*}
\end{table*}

\mypara{Computational cost.}\label{sec:geometry-complexity}
For $n$ sampled observations of $u$ units, forming the maximum-coordinate distance matrix costs $O(n^2u)$ time and $O(n^2)$ memory. Farthest-first traversal updates the nearest-center distance in $O(n^2)$ time; the same traversal supplies all reported radii. Covariance formation and eigendecomposition cost $O(nu^2+u^3)$ time and $O(u^2)$ matrix storage. Time-shift replicates repeat these operations, with shifts preserving each unit's empirical marginal. The sample cap and number of controls are fixed in~\cref{tab:hc3-cover-protocol,tab:hc3-geometry-protocol}.

\subsection{Covariance Geometry}
We compute the covariance spectrum and participation-ratio effective dimension in each block. Reference covariance eigenvectors define the projection applied to the later diagnostic block. The number of directions needed to retain a given fraction of variance measures covariance concentration; the covering exponent in~\cref{eq:state-covering-growth} controls how many representatives are needed to approximate every state at a given accuracy.

\begin{table*}[!t]
  \centering
  \caption{Covariance geometry settings. All use the same diagnostic block.}
  \label{tab:hc3-geometry-protocol}
  \small
  \setlength{\tabcolsep}{5pt}
  \renewcommand{\arraystretch}{1.08}
  \begin{tabular}{@{}p{0.23\textwidth}p{0.21\textwidth}p{0.51\textwidth}@{}}
    \toprule
    Analysis & Parameter & Setting \\
    \midrule
    Covariance geometry & Shift control & 19 replicates; circular offsets $\ge10$\,s from either endpoint \\
    & Block bootstrap & 200 replicates; 10\,s blocks; 95\% intervals \\
    \bottomrule
  \end{tabular}
\end{table*}

The controls independently circularly shift each unit's observation sequence, preserving its marginal distribution and circular autocorrelation while changing alignment between units. The comparison reflects shared temporal modulation and nonstationarity as well as coordination between units. We compare all controls within each session and assess uncertainty by resampling contiguous time blocks. Complete session results and window sensitivity appear in~\cref{tab:hc3-geometry-sessions,tab:hc3-window-sensitivity}.

The reference subspace captures substantial variance in the later block; the covariance summaries retain the complete cohort and window comparisons.

\begin{table*}[!t]
  \centering
  \caption{Complete primary-window results. Effective dimension and shift reduction refer to the diagnostic block. The 90\% rank uses directions from the reference basis.}
  \label{tab:hc3-geometry-sessions}
  \small
  \begin{tabular}{@{}lrrrr@{}}
    \toprule
    Session & Units & Effective Dim. & Shift Reduction (\%) & 90\% Rank \\
    \midrule
    ec012ec.187 & 20 & 9.2 & 23.0 & 14 \\
    ec013.156 & 56 & 17.8 & 25.1 & 39 \\
    ec013.198 & 61 & 19.8 & 27.9 & 47 \\
    ec013.250 & 37 & 19.8 & 10.7 & 27 \\
    ec013.375 & 38 & 17.1 & 29.7 & 27 \\
    ec013.406 & 47 & 20.8 & 13.1 & 33 \\
    ec013.553 & 52 & 26.2 & 13.1 & 36 \\
    ec013.646 & 37 & 19.9 & 11.8 & 27 \\
    ec014.123 & 57 & 19.0 & 24.3 & 31 \\
    ec014.183 & 104 & 26.4 & 42.4 & 63 \\
    ec014.195 & 91 & 32.7 & 29.1 & 53 \\
    ec014.215 & 87 & 28.6 & 37.0 & 52 \\
    ec014.260 & 59 & 28.8 & 21.3 & 42 \\
    ec014.277 & 69 & 34.4 & 19.1 & 50 \\
    ec014.723 & 35 & 16.4 & 18.6 & 24 \\
    ec015.043 & 48 & 18.2 & 26.3 & 31 \\
    ec015.045 & 47 & 20.5 & 28.5 & 35 \\
    ec015.047 & 46 & 23.9 & 18.9 & 33 \\
    ec016.233 & 52 & 17.4 & 25.6 & 32 \\
    ec016.269 & 57 & 15.6 & 25.0 & 37 \\
    ec016.397 & 51 & 23.7 & 12.9 & 36 \\
    ec016.430 & 44 & 23.9 & 8.2 & 30 \\
    ec016.448 & 43 & 22.6 & 5.0 & 29 \\
    ec016.483 & 40 & 19.2 & 16.7 & 29 \\
    ec016.582 & 61 & 30.8 & 17.0 & 43 \\
    \bottomrule
  \end{tabular}
\end{table*}

\begin{table}[htbp]
  \centering
  \caption{Window sensitivity. Entries are session medians in the diagnostic block; reduction is relative to the median shift control. Every eligible session lies below every shift control at each window length.}
  \label{tab:hc3-window-sensitivity}
  \small
  \begin{tabular}{@{}lrrr@{}}
    \toprule
    Window & Eff. Dim. & Unit Fraction & Reduction \\
    \midrule
    .5\,s & 23.76 & 50.33\% & 19.27\% \\
    1.0\,s & 20.50 & 45.77\% & 21.33\% \\
    2.0\,s & 17.58 & 36.47\% & 25.04\% \\
    \bottomrule
  \end{tabular}
\end{table}

\raggedbottom
\section{HC-3 Prediction Details}\label{sec:hc3-prediction-details}

\mypara{Matched prediction task.}

We use the same metadata-selected session panel and preprocessing as~\cref{sec:hc3-geometry-details}. Units are partitioned into input queries $Q$, candidate target queries, and an independent evaluation set $V$. Nested target budgets share the input set and evaluation units within each seed. Training uses neural observations and unit identities. A cross-attention \textit{encoder} compresses the input observations into learned tokens, and a query-conditioned \textit{predictor} estimates target observations. All parameters are trained jointly from scratch. The full design and fixed endpoint are listed in~\cref{tab:hc3-b-design}.

\begin{table}[!htbp]
  \centering
  \caption{HC-3 prediction protocol. Percentages denote within-session chronological blocks or unit partitions. Normalization uses the earlier training portion specified in the geometry protocol. All configurations use the fixed final epoch.}
  \label{tab:hc3-b-design}
  \small
  \setlength{\tabcolsep}{4pt}
  \begin{tabular*}{\columnwidth}{@{\extracolsep{\fill}}lr@{}}
    \toprule
    Design & Value \\
    \midrule
    Sessions / paired seeds & $25$ / $3$ \\
    $Q$ / target pool / $V$ & $40\%/40\%/20\%$ \\
    Target pool fractions & $1/8,1/4,1/2,1$ \\
    Configurations & $300$ \\
    Training block & $0$--$45\%$ \\
    Past reference block & $45$--$65\%$ \\
    Validation block & $65$--$75\%$ \\
    Test block & $75$--$100\%$ \\
    Window / stride / guard & $1/.25/2$ s \\
    Hidden width / tokens & $136/48$ \\
    \textit{Encoder} / \textit{predictor} depth & $4/4$ \\
    Attention heads / FF width & $8/544$ \\
    Dropout / batch & $.05/64$ \\
    Epochs & $200$ \\
    Cosine learning rate & $3\times10^{-4}\to3\times10^{-6}$ \\
    \bottomrule
  \end{tabular*}
\end{table}

\mypara{Training choices.}\label{sec:hc3-training-choices}
The reference architecture and $200$-epoch endpoint were fixed after preliminary training checks and then shared by every session, target budget, and repetition. The small model reduces depth, width, and learned-token count while retaining the same observation task. These architecture settings are design choices, not quantities selected separately on the final test block. AdamW uses weight decay $10^{-4}$, moments $(.9,.999)$, and stabilizer $10^{-8}$; gradients are clipped to norm $1$. Training uses Float32, with TF32 disabled and deterministic PyTorch operations. The loss is the batch mean of the maximum absolute target error. The cosine schedule, chronological split, and fixed final epoch are given in~\cref{tab:hc3-b-design}; all remaining initialization and library settings are specified in the archived implementation.

\mypara{Attention cost.}\label{sec:hc3-attention-cost}
Let $q=|Q|$, $p=|P|$, $t$ be the learned-token count, and $h$ the hidden width. Per observation and layer, the dense attention and projection operations cost $O((q+p+t)h^2+(t^2+tq+p^2+pt)h)$ time. Materialized attention matrices require $O(t^2+tq+p^2+pt)$ storage for a fixed number of heads, alongside the token activations. Batch size and layer count multiply these costs; the unit-embedding table adds storage linear in the number of recorded units and $h$. The tables specify each architecture and the supplied manifests give the session-dependent query counts.

\mypara{Prediction score.}
The prediction score averages the maximum absolute target-observation error across test windows. The constant baseline is the coordinatewise training median on the same target set. Relative gain is one minus the ratio of model error to baseline error. We average seed-level gains within each session, then weight sessions equally. Session-to-session spread is descriptive because recordings within an animal share biological variation. Different budgets evaluate different target sets, each with its own matched baseline.

\begin{table}[!htbp]
  \centering
  \caption{Prediction gain, mean $\pm$ sample standard deviation over repetitions after averaging the same sessions. Each target set uses its own training-median baseline.}
  \label{tab:hc3-prediction}
  \small
  \setlength{\tabcolsep}{2pt}
  \begin{tabular*}{\columnwidth}{@{\extracolsep{\fill}}lr@{}}
    \toprule
    Targets & Prediction Gain \\
    \midrule
    $1/2$ & $7.0\pm1.0$\% \\
    Full & $21.0\pm1.1$\% \\
    \bottomrule
  \end{tabular*}
\end{table}

\section{Local State Recovery}\label{sec:hc3-local-details}

\mypara{Reference state.}
The ontic-state surrogate is the large \textit{encoder}'s complete post-normalization token state, flattened into a reference vector. For each session and seed, the full-target large model is held fixed. The smaller \textit{encoder--predictor} learns the same input-to-target task from scratch, with identical input and target queries, preprocessing, training observations, and chronological validation and test blocks. Its training objective uses target observations. After predictive training, the reference states determine the recovery maps and their evaluation. We compare the final small \textit{encoder} with its exact saved initialization on the same held-out inputs.

\begin{table}[htbp]
  \centering
  \caption{Matched prediction models. The common optimization schedule and chronological split are given in~\cref{tab:hc3-b-design}. Recovery uses the full target pool, with $25$ sessions and $3$ paired seeds. Parameter counts include the \textit{encoder} and \textit{predictor}.}
  \label{tab:local-models}
  \small
  \setlength{\tabcolsep}{4pt}
  \begin{tabular*}{\columnwidth}{@{\extracolsep{\fill}}lrr@{}}
    \toprule
    Configuration & \textit{Small} & \textit{Reference} \\
    \midrule
    \textit{Encoder} / \textit{predictor} depth & $2/2$ & $4/4$ \\
    Hidden width / tokens & $64/16$ & $136/48$ \\
    Attention heads / FF width & $4/256$ & $8/544$ \\
    State coordinates & $1{,}024$ & $6{,}528$ \\
    Neighbor candidates & $1,4,16,64$ & --- \\
    Fixed training epoch & $200$ & $200$ \\
    Parameters (M) & $\sim .3$ & $\sim 2.4$ \\
    \bottomrule
  \end{tabular*}
\end{table}

\mypara{Local recovery.}
For a held-out input, we locate the nearest neighbors of its learned state in the small \textit{encoder}'s training-state library and average their reference states. Distances use Euclidean coordinates, centered by the training mean and divided by a common training root-mean-square scale. This scalar normalization preserves neighbor order. A neighbor's weight is proportional to the exponential of its negative squared distance divided by the squared distance of the furthest selected neighbor, with a numerical floor. At test time, these training neighbors and their distance-based weights directly determine the recovered reference state.

\mypara{Selection and cost.}\label{sec:local-reproduction}
The neighborhood grid spans individual neighbors through broader local averages. The count minimizing validation mean squared error is selected separately for the trained and initial representations; ties retain the first count in ascending order. Evaluation uses the fixed final test block. For a library of $n$ training states with $s$ learned coordinates, a batch of $B$ queries requires $O(Bns)$ distance computation. Neighbor selection adds the cost of selecting the smallest $k$ distances per query; averaging reference vectors of dimension $r$ costs $O(Bkr)$. Batches of $128$ bound the distance workspace to $O(128n)$ entries, in addition to stored learned and reference states. No gradient updates are used to fit this local recovery map.

\mypara{Recovery score.}
We score all reference coordinates together using squared reconstruction error divided by the squared error of the training-mean reference state; the reported $R^2$ is one minus this ratio. Both conditions use this same baseline. We first average paired seeds within each session and then weight sessions equally. Seeds change the query partition and the independently trained model pair. Animal groups follow recording metadata; recording-name prefixes can differ from animal identities.

\begin{table}[!htbp]
  \centering
  \caption{Local-recovery gains, mean $\pm$ sample standard deviation over repetitions with the session or animal weights held fixed.}
  \label{tab:local-summary}
  \small
  \setlength{\tabcolsep}{2pt}
  \begin{tabular*}{\columnwidth}{@{\extracolsep{\fill}}lr@{}}
    \toprule
    Aggregation & $\Delta R^2$ \\
    \midrule
    All sessions & $.0239\pm.0118$ \\
    ec012 & $.0622\pm.0700$ \\
    ec013 & $.0479\pm.0169$ \\
    ec014 & $.0038\pm.0145$ \\
    ec016 & $.0233\pm.0276$ \\
    Animal-equal & $.0343\pm.0218$ \\
    \bottomrule
  \end{tabular*}
\end{table}

\end{document}